\documentclass[letterpaper, 10pt, conference]{IEEEtran}

\IEEEoverridecommandlockouts

\usepackage{hyperref, booktabs}
\hypersetup{
    colorlinks,
    linkcolor={blue!50!black},
    citecolor={blue!50!black},
    urlcolor={red!50!black}
}

\usepackage{amsmath,amssymb,amsthm}
\usepackage{euscript}
\usepackage[lined, ruled, boxed, noend]{algorithm2e}
\usepackage{tikz,pgfplots}
\pgfplotsset{compat=1.14}

\usepackage[]{todonotes}
\newtheorem{proposition}{Proposition}
\newtheorem{remark}{Remark}

\theoremstyle{definition}
\newtheorem{definition}{Definition}
\newtheorem{problem}{Problem}

\usetikzlibrary{arrows,calc,decorations.markings,positioning,automata,hobby}
\tikzset{
    block/.style = {draw, rectangle,
        minimum height=1cm,
        minimum width=2cm},
    input/.style = {coordinate,node distance=1cm},
    output/.style = {coordinate,node distance=4cm},
    arrow/.style={draw, -latex,node distance=2cm},
    pinstyle/.style = {pin edge={latex-, black,node distance=2cm}},
    sum/.style = {draw, circle, node distance=1cm},
}

\title{\LARGE \bf Intermittent Connectivity for Exploration in Communication-Constrained Multi-Agent Systems}

 \author{Filip Klaesson, Petter Nilsson, Aaron D. Ames and Richard M. Murray
 \thanks{The authors are with California Institute of Technology, Pasadena, CA.}
 }

\begin{document}

\maketitle
\thispagestyle{empty}
\pagestyle{empty}

\begin{abstract}
    Motivated by exploration of communication-constrained underground environments using robot teams, we study the problem of planning for intermittent connectivity in multi-agent systems. We propose a novel concept of information-consistency to handle situations where the plan is not initially known by all agents, and suggest an integer linear program for synthesizing information-consistent plans that also achieve auxiliary goals. Furthermore, inspired by network flow problems we propose a novel way to pose connectivity constraints that scales much better than previous methods. In the second part of the paper we apply these results in an exploration setting, and propose a clustering method that separates a large exploration problem into smaller problems that can be solved independently. We demonstrate how the resulting exploration algorithm is able to coordinate a team of ten agents to explore a large environment.
\end{abstract}

\section{Introduction}

In many multi-agent applications such as mapping, frontier exploration, and search-and-rescue missions, communication between agents is critical to gain awareness of the situation. Significant effort has been devoted to study connectivity in continuous models  \cite{Zavlanos2007,DeGennaro2008,mesbahi2010graph}. Although such models are closer to the actual dynamics of robots, solution techniques become difficult to apply in complex geometries. This paper considers an alternative approach of planning on a discrete graph model that abstracts away low-level dynamics but, in our opinion, retains the core elements of the connectivity problem.

Various solutions have been proposed to maintain connectivity between agents in environments like tunnel and cave systems that suffer from severe communication constraints \cite{Jensen2018, Hollinger2012, Liu2017}. The fundamental difference between these methods is the type of connectivity constraint enforced between agents. 
The situational awareness of agents increases with the strictness of the constraint while a more relaxed constraint allows them to perform individual tasks to a higher degree. 

The strictest connectivity constraint is the \emph{continuous} connectivity constraint which demands that all agents are fully connected at all times \cite{ROOKER2007435}. This implies that two arbitrary agents in the network can share data with each other at any time. Although this results in global awareness, it can degrade the objective performance significantly. \emph{Recurrent} connectivity \cite{Banfi2018} is a relaxation of continuous connectivity that allows agents to occasionally disconnect and instead enforces global connectivity only at a specific time instant (such as the final time), which allows more flexibility at the cost of awareness.

However, connectivity constraints like continuous and recurrent connectivity that occasionally enforce global connectivity do not perform well in exploration tasks if the environment is large relative to the number of agents. For instance, the exploration setting studied in \cite{Banfi2018} incorporates a fixed base station that needs to be updated of the exploration progress. In this case the recurrent connectivity constraint implies that agents should be connected to the base station even when they reach frontiers. If a frontier is further away than the reach of the longest fully connected configuration it is therefore not feasible to explore it, and even if the frontier is within reach recurrent connectivity may be inefficient since a large number of agents may be required to establish connectivity and therefore can not partake in exploration. 

Motivated by exploration settings where the environment is large in relation to the number of robots we propose a novel \emph{intermittent} connectivity constraint that formalizes data distribution behaviors. 
Intermittent connectivity does not ever impose simultaneous global connectivity; it only requires a way to pass data between certain agents within some time horizon. This allows agents to transport data and distribute it between each other in a collaborative manner. A stricter version of intermittent connectivity was formulated via linear temporal logic in \cite{Kantaros2017}, together with a way to generate trajectories. Our formulation is more permissive as it does not pre-suppose a division of agents into teams, and does not require all members of a team to meet repeatedly at the same location. Instead we allow for targeted information distribution which leaves more flexibility to perform actual exploration tasks. 
We also consider the case when the plan itself is part of the data that needs to be distributed, i.e., agents that have not yet received this data are not aware of the plan and can therefore not be expected to act in accordance with it. We handle this situation by designating a \emph{master} agent with initial knowledge of the plan, and pose \emph{master constraints} that ensure that the plan is consistent with information distribution constraints.

As previous work \cite{Banfi2018,Rossi2018ReviewOM,Kara2006IntegerLP,Earl2002ModelingAC,CTL18} we plan multi-agent trajectories via an Integer Linear Program (ILP). Connectivity problems have also been solved with the (mixed) ILP approach: \cite{Chatzipanagiotis2016} proposed a MILP that keeps static source states connected over time. Inspired by ideas from network flows we introduce new linear constraints that guarantee our relaxed intermittent connectivity criterion and show that this formulation has better scalability properties than those found in previous work; the number of constraints scales linearly with the graph size instead of combinatorially. However, since ILPs are NP-hard even the improved flow formulation encounters scalability issues for large graphs. To mitigate this issue we propose a clustering method tailored to exploration problems that decomposes a large problem into smaller problems that can be solved more efficiently.


\section{Preliminaries and Problem Formulation}
\label{sec:prelims}

In this section we first present preliminaries in order to state the intermittent connectivity problem that we seek to solve.

\subsection{Environment model}

When constructing the environment model the possibility of communication between locations needs to be considered. Determining whether communication is possible between two locations could be done via models, data, or both. In order to not restrict the framework we do not specify a communication condition and instead use an abstract mobility-communication-network as a model of the environment.

\begin{definition} 
A mobility-communication-network $\EuScript{N}$ is defined as the tuple $\EuScript{N}=(S,\rightarrow,\rightsquigarrow)$:
\begin{itemize}
    \item $S$: finite set of states $s$,
    \item $\rightarrow \subset S \times S$: set of directed mobility edges $(s,s')$,
    \item $\rightsquigarrow  \subset S \times S$: set of directed communication edges $(s,s')$.
\end{itemize}
\label{def:env}
\end{definition}

The states $s\in S$ correspond to locations in the environment and a mobility edge $(s,s')\in\rightarrow$ denotes a directed traversable path from $s$ to $s'$. The existence of a communication edge $(s,s')\in\rightsquigarrow$ certifies that an agent located in $s$ can send data to an agent located in $s'$. We assume that agents in the same state are able to share information with each other. It is also assumed that $\rightsquigarrow$ does not contain false positives.

For optimization purposes we introduce edge weights $C$ and $\widetilde C$ that associate weights with the edges in $\rightarrow$ and $\rightsquigarrow$, respectively. Likewise, state rewards $\mathfrak R$ are used to associate a reward to reaching certain locations. We seek to find finite-time trajectories and denote the time horizon by T.

A key component to formalize the intermittent connectivity constraint is the time-extended version of the mobility-communication network, which is built by stacking copies of the network and associating each layer with a time step. The mobility edges in each layer are rewired across time steps, as illustrated in Fig \ref{fig:time_extended-graph}. In the time-extended graph agents therefore traverse over time steps but communicate within time steps.

\begin{figure}[t!]
    \centering
    \resizebox{!}{0.7\columnwidth}{%
            \begin{tikzpicture}[>=stealth,node distance=2.7cm,auto]
            \tikzstyle{agent}=[circle,thick,draw=blue!75,fill=blue!20,minimum size=8mm]
        
            \node[state]                    (s0)                    {$s_0$};
            \node[state]                    (s1) [above of = s0]      {$s_1$};
            \node[state]                    (s2) [above of = s1]       {$s_2$};
            \node[state]                    (s3) [above of = s2]      {$s_3$};
                     
            \draw[dashed,->] (s0) to [out=60,in=-60] (s1);
            \draw[dashed,->] (s1) to [out=60,in=-60] (s2);
            \draw[dashed,->] (s2) to [out=60,in=-60] (s3);
            \draw[dashed,<-] (s0) to [out=120,in=-120] (s1);
            \draw[dashed,<-] (s1) to [out=120,in=-120] (s2);
            \draw[dashed,<-] (s2) to [out=120,in=-120] (s3);
            
            \path[->]
                (s0) edge   node {}     (s1)
                (s1) edge   node {}     (s0)
                (s1) edge   node {}     (s2)
                (s2) edge   node {}     (s1)
                (s2) edge   node {}     (s3)
                (s3) edge   node {}     (s2)
                (s0) edge [loop right] node {} (s0)
                (s1) edge [loop right] node {} (s1)
                (s2) edge [loop right] node {} (s2)
                (s3) edge [loop right] node {} (s3);
                
            \node at ([xshift=5mm]s0) {\includegraphics[width=0.8cm]{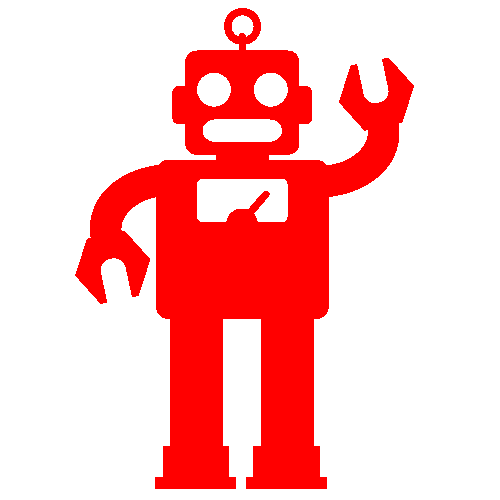}};
            \node at ([xshift=5mm]s1){\includegraphics[width=0.8cm]{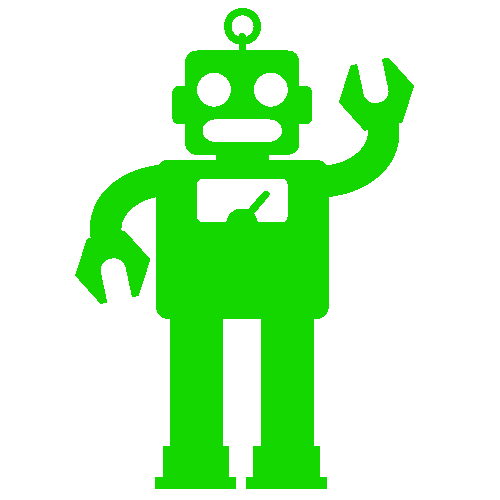}};
            \node at ([xshift=5mm]s3){\includegraphics[width=0.8cm]{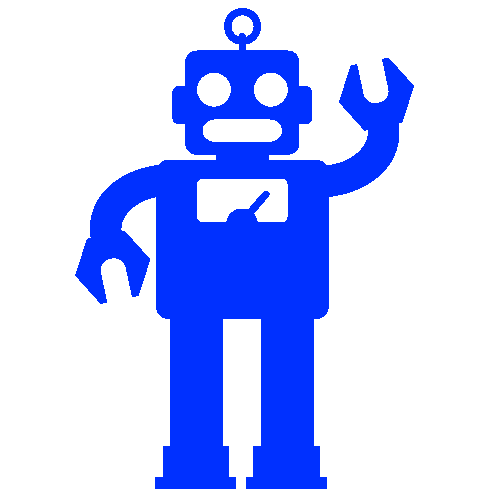}};

            \node[] at ([yshift=-9mm]s0) { };

            \node[below of=s0, yshift=12mm] {\large $\EuScript{N}$};
            \end{tikzpicture}
    }
    ~ \hspace{5mm}
    \resizebox{!}{0.7\columnwidth}{%
            \begin{tikzpicture}[>=stealth,node distance=2.7cm,auto]
            \tikzstyle{agent}=[circle,thick,draw=blue!75,fill=blue!20,minimum size=5mm]
        
            \node[state]                    (s0t0)                    {$s_0t_0$};
            \node[state]                    (s1t0) [above of = s0t0]      {$s_1t_0$};
            \node[state]                    (s2t0) [above of = s1t0]       {$s_2t_0$};
            \node[state]                    (s3t0) [above of = s2t0]      {$s_3t_0$};
                     
            \draw[line width=1mm, red, dashed,->] (s0t0) to [out=60,in=-60] (s1t0);
            \draw[dashed,->] (s1t0) to [out=60,in=-60] (s2t0);
            \draw[dashed,->] (s2t0) to [out=60,in=-60] (s3t0);
            \draw[dashed,<-] (s0t0) to [out=120,in=-120] (s1t0);
            \draw[dashed,<-] (s1t0) to [out=120,in=-120] (s2t0);
            \draw[dashed,<-] (s2t0) to [out=120,in=-120] (s3t0);
        
            \node[state]                    (s0t1) [right of = s0t0]      {$s_0t_1$};
            \node[state]                    (s1t1) [above of = s0t1]      {$s_1t_1$};
            \node[state]                    (s2t1) [above of = s1t1]       {$s_2t_1$};
            \node[state]                    (s3t1) [above of = s2t1]      {$s_3t_1$};

            \draw[dashed,->] (s0t1) to [out=60,in=-60] (s1t1);
            \draw[dashed,->] (s1t1) to [out=60,in=-60] (s2t1);
            \draw[line width=1mm, red, dashed, dashed,->] (s2t1) to [out=60,in=-60] (s3t1);
            \draw[dashed,<-] (s0t1) to [out=120,in=-120] (s1t1);
            \draw[dashed,<-] (s1t1) to [out=120,in=-120] (s2t1);
            \draw[line width=1mm, blue, dashed, dashed,<-] (s2t1) to [out=120,in=-120] (s3t1);

            \node[state]                    (s0t2) [right of = s0t1]    {$s_0t_2$};
            \node[state]                    (s1t2) [above of = s0t2]    {$s_1t_2$};
            \node[state]                    (s2t2) [above of = s1t2]    {$s_2t_2$};
            \node[state]                    (s3t2) [above of = s2t2]     {$s_3t_2$};
                     
            \draw[dashed,->] (s0t2) to [out=60,in=-60] (s1t2);
            \draw[dashed,->] (s1t2) to [out=60,in=-60] (s2t2);
            \draw[dashed,->] (s2t2) to [out=60,in=-60] (s3t2);
            \draw[line width=1mm, blue, dashed, dashed,<-] (s0t2) to [out=120,in=-120] (s1t2);
            \draw[dashed,<-] (s1t2) to [out=120,in=-120] (s2t2);
            \draw[dashed,<-] (s2t2) to [out=120,in=-120] (s3t2);

            \path[->]
                (s0t0) edge   node {}     (s0t1)
                (s0t0) edge   node {}     (s1t1)
                (s1t0) edge   node {}     (s0t1)
                (s1t0) edge   node {}     (s1t1)
                (s1t0) edge   node {}     (s2t1)
                (s2t0) edge   node {}     (s1t1)
                (s2t0) edge   node {}     (s2t1)
                (s2t0) edge   node {}     (s3t1)
                (s2t1) edge  node {}     (s1t2)
                (s3t0) edge   node {}     (s2t1)
                (s3t0) edge   node {}     (s3t1)
                (s0t1) edge   node {}     (s0t2)
                (s0t1) edge   node {}     (s1t2)
                (s1t1) edge   node {}     (s0t2)
                (s1t1) edge   node {}     (s1t2)
                (s1t1) edge   node {}     (s2t2)
                (s2t1) edge   node {}     (s1t2)
                (s2t1) edge   node {}     (s2t2)
                (s2t1) edge   node {}     (s3t2)
                (s3t1) edge   node {}     (s2t2)
                (s3t1) edge   node {}     (s3t2);
                
            \path[line width=0.5mm, red, ->]
                (s0t0) edge  node {}     (s0t1)
                (s0t1) edge  node {}     (s0t2);
                
            \path[line width=0.5mm, green, ->]
                (s1t0) edge  node {}     (s2t1)
                (s2t1) edge  node {}     (s1t2);
                
            \path[line width=0.5mm, blue, ->]
                (s3t0) edge  node {}     (s3t1)
                (s3t1) edge  node {}     (s3t2);

            \node at ([xshift=5mm]s0t0) {\includegraphics[width=0.8cm]{figures/red.png}};
            \node at ([xshift=5mm]s1t0){\includegraphics[width=0.8cm]{figures/green.png}};
            \node at ([xshift=5mm]s3t0){\includegraphics[width=0.8cm]{figures/blue.png}};
    
            \node at ([xshift=5mm]s0t1) {\includegraphics[width=0.8cm]{figures/red.png}};
            \node at ([xshift=5mm]s2t1){\includegraphics[width=0.8cm]{figures/green.png}};
            \node at ([xshift=5mm]s3t1){\includegraphics[width=0.8cm]{figures/blue.png}};
    
            \node at ([xshift=5mm]s0t2) {\includegraphics[width=0.8cm]{figures/red.png}};
            \node at ([xshift=5mm]s1t2){\includegraphics[width=0.8cm]{figures/green.png}};
            \node at ([xshift=5mm]s3t2){\includegraphics[width=0.8cm]{figures/blue.png}};
    
            \node at ([yshift=-8mm]s0t0) {t = 0};
            \node at ([yshift=-8mm]s0t1) {t = 1};
            \node at ([yshift=-8mm]s0t2) {t = 2};

            \node[below of=s0t1, yshift=12mm] {\large Time-extended graph of $\EuScript{N}$};
            
            \end{tikzpicture}
    }
\vspace{-2mm}
\caption{Mobility-communication-network $\EuScript{N}$ and the time-extended graph of $\EuScript{N}$ for $T = 2$. Mobility edges are solid and communication edges are dashed. A solution for the problem of sharing data between the blue and red agents are shown in the time-extended graph. The colors of the communication edges indicate the origin agent of the data that is transmitted.}
\label{fig:time_extended-graph}
\vspace{-3mm}
\end{figure}

We introduce some simplifying notation. Given a mobility-communication-network $\EuScript{N}=(S,\rightarrow,\rightsquigarrow)$, denote the set of mobility and communication predecessors of state $s$ as $C_{\rightarrow}^-(s) = \{s'\mid (s',s) \in \rightarrow\}$ and $C_{\rightsquigarrow}^-(v) = \{s'\mid (s',s) \in \rightsquigarrow\}$. Likewise, denote the set of mobility and communication successor of state $s$ as $C_{\rightarrow}^+(s) = \{s'\mid (s,s') \in \rightarrow\}$ and $C_{\rightsquigarrow}^+(v) = \{s'\mid (s,s') \in \rightsquigarrow\}$. It's also useful to look at both mobility and communication predecessors and successors; denote the set of all predecessors by $C^-(s) = C_{\rightarrow}^-(s) \cup C_{\rightsquigarrow}^-(s)$ and the set of all successors by $C^+(s) = C_{\rightarrow}^+(s) \cup C_{\rightsquigarrow}^+(s)$.

\subsection{Intermittent connectivity}

An intermittent connection allows the agents to disconnect, move around, and establish connections to other agents in the network. Two agents $1$ and $2$ can therefore share data with each other without ever being directly connected if other agents set up a sequence of intermittent connections such that data is transferred from $1$ to $2$, and vice versa. A simple example is shown in Fig. \ref{fig:time_extended-graph} where the red and blue agents share information with each other using a sequence of intermittent connections via the green agent. We seek to formalize this type of intermittent connectivity and start by defining a \emph{mobility path} which encodes an agent trajectory.

\begin{definition}
Given a mobility-communication-network $\EuScript{N}=(S,\rightarrow,\rightsquigarrow)$, a \emph{mobility path} is a sequence $\pi$ : $s_0s_1s_2 \ldots s_T$ of states $s\in S$ such that $(s_t,s_{t+1})\in\rightarrow$. 
\end{definition}

We use a superscript $r$ to specify the mobility path for agent $r$ by $\pi^r = s_0^rs_1^rs_2^r \ldots s_T^r$. Information can be transferred spatially via two mechanisms: \textit{(i)} agents can traverse mobility edges and carry data with them, and \textit{(ii)} agents can transmit information to other agents over communication edges. These two types of information transfer guides us to the notion of an \emph{information path}.

\begin{definition}
Given a mobility-communication network $\EuScript{N}=(S,\rightarrow,\rightsquigarrow)$, an \emph{information path} is a sequence $\bar{\pi}$ : $\mathfrak s_0^0 \mathfrak s_0^1 \ldots \mathfrak s_0^{n_0} \mathfrak s_1^0 \mathfrak s_1^1 \ldots \mathfrak s_1^{n_1} \ldots \mathfrak s_T^{n_T}$ of states $\mathfrak s_t^i \in S$ such that $(\mathfrak s_t^{n_t},\mathfrak s_{t+1}^0)\in\rightarrow$ and $(\mathfrak s_t^i,\mathfrak s_t^{i+1})\in\rightsquigarrow$. In addition, an \emph{agent-to-agent information path} $\bar{\pi}^{i,j}$ is an information path that begins in the initial state of agent $i$ and ends in the final state of agent $j$, i.e., $\mathfrak s_0^0 = s_0^i$ and $\mathfrak s_T^{n_T} = s^j_T$.
\end{definition}

\begin{figure}[t!]
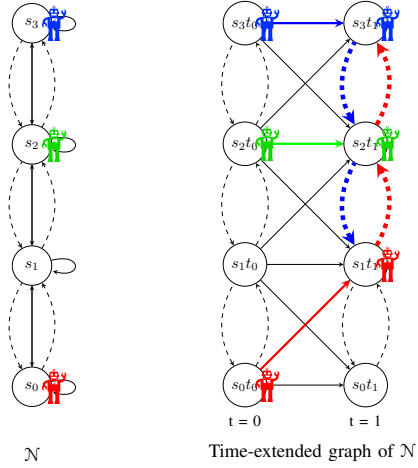

    \centering
    \resizebox{!}{0.7\columnwidth}{%
            \begin{tikzpicture}[>=stealth,node distance=2.7cm,auto]
            \tikzstyle{agent}=[circle,thick,draw=blue!75,fill=blue!20,minimum size=8mm]
        
            \node[state]                    (s0)                    {$s_0$};
            \node[state]                    (s1) [above of = s0]      {$s_1$};
            \node[state]                    (s2) [above of = s1]       {$s_2$};
            \node[state]                    (s3) [above of = s2]      {$s_3$};
                     
            \draw[dashed,->] (s0) to [out=60,in=-60] (s1);
            \draw[dashed,->] (s1) to [out=60,in=-60] (s2);
            \draw[dashed,->] (s2) to [out=60,in=-60] (s3);
            \draw[dashed,<-] (s0) to [out=120,in=-120] (s1);
            \draw[dashed,<-] (s1) to [out=120,in=-120] (s2);
            \draw[dashed,<-] (s2) to [out=120,in=-120] (s3);
            
            \path[->]
                (s0) edge   node {}     (s1)
                (s1) edge   node {}     (s0)
                (s1) edge   node {}     (s2)
                (s2) edge   node {}     (s1)
                (s2) edge   node {}     (s3)
                (s3) edge   node {}     (s2)
                (s0) edge [loop right] node {} (s0)
                (s1) edge [loop right] node {} (s1)
                (s2) edge [loop right] node {} (s2)
                (s3) edge [loop right] node {} (s3);
                
            \node at ([xshift=5mm]s0) {\includegraphics[width=0.8cm]{figures/red.png}};
            \node at ([xshift=5mm]s2){\includegraphics[width=0.8cm]{figures/green.png}};
            \node at ([xshift=5mm]s3){\includegraphics[width=0.8cm]{figures/blue.png}};
            
            \node at ([yshift=-8mm]s0) {};
             
            \node[below of=s0, yshift=12mm] {\large $\EuScript{N}$};

            \end{tikzpicture}
    }
    ~ \hspace{1cm}
    \resizebox{!}{0.7\columnwidth}{%
            \begin{tikzpicture}[>=stealth,node distance=2.7cm,auto]
            \tikzstyle{agent}=[circle,thick,draw=blue!75,fill=blue!20,minimum size=5mm]
        
            \node[state]                    (s0t0)                    {$s_0t_0$};
            \node[state]                    (s1t0) [above of = s0t0]      {$s_1t_0$};
            \node[state]                    (s2t0) [above of = s1t0]       {$s_2t_0$};
            \node[state]                    (s3t0) [above of = s2t0]      {$s_3t_0$};
                     
            \draw[dashed,->] (s0t0) to [out=60,in=-60] (s1t0);
            \draw[dashed,->] (s1t0) to [out=60,in=-60] (s2t0);
            \draw[dashed,->] (s2t0) to [out=60,in=-60] (s3t0);
            \draw[dashed,<-] (s0t0) to [out=120,in=-120] (s1t0);
            \draw[dashed,<-] (s1t0) to [out=120,in=-120] (s2t0);
            \draw[dashed,<-] (s2t0) to [out=120,in=-120] (s3t0);
        
            \node[state]                    (s0t1) [right of = s0t0]      {$s_0t_1$};
            \node[state]                    (s1t1) [above of = s0t1]      {$s_1t_1$};
            \node[state]                    (s2t1) [above of = s1t1]       {$s_2t_1$};
            \node[state]                    (s3t1) [above of = s2t1]      {$s_3t_1$};

            \draw[dashed,->] (s0t1) to [out=60,in=-60] (s1t1);
            \draw[line width=1mm, red, dashed,->] (s1t1) to [out=60,in=-60] (s2t1);
            \draw[line width=1mm, red, dashed, dashed,->] (s2t1) to [out=60,in=-60] (s3t1);
            \draw[dashed,<-] (s0t1) to [out=120,in=-120] (s1t1);
            \draw[line width=1mm, blue, dashed,<-] (s1t1) to [out=120,in=-120] (s2t1);
            \draw[line width=1mm, blue, dashed, dashed,<-] (s2t1) to [out=120,in=-120] (s3t1);

            \path[->]
                (s0t0) edge   node {}     (s0t1)
                (s0t0) edge   node {}     (s1t1)
                (s1t0) edge   node {}     (s0t1)
                (s1t0) edge   node {}     (s1t1)
                (s1t0) edge   node {}     (s2t1)
                (s2t0) edge   node {}     (s1t1)
                (s2t0) edge   node {}     (s2t1)
                (s2t0) edge   node {}     (s3t1)
                (s3t0) edge   node {}     (s2t1)
                (s3t0) edge   node {}     (s3t1);
                
            \path[line width=0.5mm, red, ->]
                (s0t0) edge  node {}     (s1t1);
                
            \path[line width=0.5mm, green, ->]
                (s2t0) edge  node {}     (s2t1);
                
            \path[line width=0.5mm, blue, ->]
                (s3t0) edge  node {}     (s3t1);

            \node at ([xshift=5mm]s0t0) {\includegraphics[width=0.8cm]{figures/red.png}};
            \node at ([xshift=5mm]s2t0){\includegraphics[width=0.8cm]{figures/green.png}};
            \node at ([xshift=5mm]s3t0){\includegraphics[width=0.8cm]{figures/blue.png}};
    
            \node at ([xshift=5mm]s1t1) {\includegraphics[width=0.8cm]{figures/red.png}};
            \node at ([xshift=5mm]s2t1){\includegraphics[width=0.8cm]{figures/green.png}};
            \node at ([xshift=5mm]s3t1){\includegraphics[width=0.8cm]{figures/blue.png}};
            
            \node at ([yshift=-8mm]s0t0) {t = 0};
            \node at ([yshift=-8mm]s0t1) {t = 1};
            \node[below of=s0t1, yshift=12mm, xshift=-12mm] {\large Time-extended graph of $\EuScript{N}$};
            \end{tikzpicture}
    }
\vspace{-2mm}
\caption{Mobility-communication-network $\EuScript{N}$ and the time-extended graph of $\EuScript{N}$ for $T = 1$. A solution for the problem of sharing data between the blue and red agents are shown in the time-extended graph.}
\label{fig:multi-hop-communication}
\vspace{-3mm}
\end{figure}

Note that only one mobility transition is allowed for each time step $t$, while information can potentially travel over multiple communication edges within a time step via a multi-hop link. Fig. \ref{fig:multi-hop-communication} shows an example where the red and blue agents share information with each other using multiple communication edges within a time step by sending information via the green agent. To associate information paths with agent mobility we introduce the notation of a \emph{valid} information path.

\begin{definition}
Consider a network $\EuScript{N}=(S,\rightarrow,\rightsquigarrow)$ and a collection of mobility paths  $\left\{ \pi^r: s^r_{0}s^r_{1}s^r_{2} \ldots s_T^r  \right\}$. An information path $\bar{\pi}$ : $\mathfrak s_0^0 \mathfrak s_0^1 \ldots \mathfrak s_0^{n_0} \mathfrak s_1^0 \mathfrak s_1^1 \ldots \mathfrak s_1^{n_1} \ldots \mathfrak s_T^{n_T}$ is \emph{valid} with respect to the mobility paths $\{\pi^r\}$ if the following conditions are satisfied for all $t = 0, \ldots, T$:
\begin{align}
    \label{eq:valid_req1} \left\{ \mathfrak s^i_t \mid i=0, \ldots, n_t \right\} & \subseteq \left\{ s^r_{t} \mid r = 1 , \ldots , R \right\}, \\
    \label{eq:valid_req2} (\mathfrak s_t^{n_t}, \mathfrak s_{t+1}^0) & \in \left\{ \left( s^r_{t}, s^r_{t+1} \right) \mid r = 1, \ldots, R \right\}.
\end{align}
\end{definition}
Equation \eqref{eq:valid_req1} implies that there must be agents present at all locations that are used to transmit information over communication edges, and equation \eqref{eq:valid_req2} specifies that if information moves along a mobility edge, then at least one agent must undertake that transition at the same time step.

An advantage of intermittent connectivity is that it allows for directed information transfer. In many applications we are only interested in transferring information between specific agents instead of sharing all information among all agents. A useful formulation  used throughout this work is to enforce a subset of agents to transfer their information to all agents in another subset. We denote the set of agents initially having access to important information by sources $\texttt{src} \subset \{ 1, \ldots, R \}$ and the set of agents that are to receive information by sinks $\texttt{snk} \subset \{ 1, \ldots, R \}$. We require that all agents in $\texttt{snk}$ should receive information from all agents in $\texttt{src}$, i.e. that there should exist an agent-to-agent information path  $\bar{\pi}^{i,j}$ from every agent $i \in \texttt{src}$ to each agent $j \in \texttt{snk}$. We are now ready to formulate the intermittent connectivity problem.

\begin{problem}[Intermittent Connectivity]
\label{prob:p1}
Consider a network $\EuScript{N}$ with $R$ agents, initial conditions $\{ s_0^r \}$, and time horizon $T$. Find mobility paths $\{ \pi^{r} \}$ such that for each pair $(i, j) \in \texttt{src} \times \texttt{snk}$ there exists an agent-to-agent information path $\bar \pi^{i,j}$ that is valid with respect to $\{ \pi^r \}$ and such that
\begin{equation}
    \label{eq:p1_obj}
    \mathfrak R - \sum_r \sum_{t=0}^{T-1}  C(t, s_t^r, s_{t+1}^r) -\sum_{b \in \texttt{src}} \sum_{(t, s, s') \in E_b}  \widetilde C(t, s, s')
\end{equation}
is maximized, where $\mathfrak R$ is a reward function that gives rewards for sending agents to states:
\begin{equation}
\label{eq:reward_structure}
     \mathfrak R = \sum_{s \in S} \sum_{k=1}^R \begin{cases} \mathfrak R(s, k), & \text{if} \; \sum\limits_{r=1}^R \mathbf{1}_{s_T^r = s} \geq k, \\ 
     0, & \text{otherwise},\end{cases} 
\end{equation}
where $\mathbf{1}$ is the indicator function, and $E_b$ contains communication edges used in information paths from source $b$:
\begin{equation}
\label{eq:edge_set}
    E_b = \left\{ (t, \mathfrak s_t^{k}, \mathfrak s_t^{k+1}) \in \rightsquigarrow \; \mid \mathfrak s_t^k, \mathfrak s_t^{k+1} \in \bar \pi^{b,j}, j \in \texttt{snk} \right\}.
\end{equation}
\end{problem}
The objective function \eqref{eq:p1_obj} penalizes agent transitions and information transmission using the edge weights $C$ and $\widetilde{C}$, while rewarding the final state of agents. In some applications it can be beneficial to send multiple agents which is why the state reward of reaching a state $s$ in the final time step depend on the number of agents that reach $s$. Note that collecting reward $\mathfrak R(s,k)$ for sending $k$ agents to $s$ implies that all rewards $\mathfrak R(s,i)$ for $i < k$ are also received.

\subsection{Information-Consistent Intermittent Connectivity}
\label{ssec:ic_intermitt_conn}

A potential issue with Problem \ref{prob:p1} is that all agents may not be in communication range when a new plan is to be executed. If the optimal solution to Problem \ref{prob:p1} is unique, all agents could solve the problem and arrive at the same solution, but this does not work if the solution is non-unique or if not all agents are computationally equipped to solve the  problem. In this scenario a solution could be to solve a single instance of Problem \ref{prob:p1} and then distribute the resulting plan to all agents via intermittent connectivity. However, there are two effects that need to be considered when planning in this fashion: agents that have not yet received the plan can not be expected to \textit{(i)} move from their initial position and \textit{(ii)} transmit information to other agents. These restrictions need to be accounted for in the planning step. We assume that there is a \emph{master} agent that computes the plan and thus knows it at $t=0$, but that other agents can not act until they have received information from the master agent. In the following, information paths originating from the master agent are denoted $^m \bar {\pi}$.

\begin{definition}
\label{def:information_consistent_mob}
Consider a mobility-communication-network $\EuScript{N}=(S,\rightarrow,\rightsquigarrow)$ and an information path $^m\bar{\pi} $. A mobility path $\pi = s_0 s_1 s_2 \ldots s_T$ is \emph{information-consistent} with respect to $^m\bar \pi$ if for all $t = 0, \ldots, T$ it holds that
\begin{align}
\label{eq:information_consistent_mob}
    &s_0 \not \in \left\{\mathfrak s_j^{i} \in ^m\bar{\pi} \mid 0 \leq j \leq t, \; 0 \leq i \leq n_j \right\} \implies s_t = s_0.
\end{align}
Similarly, an information path $\bar{\pi} = \mathfrak s_0^0 \mathfrak s_0^1 \ldots \mathfrak s_0^{n_0} \ldots s_T^{n_T}$ is \emph{information-consistent} with $^m\bar \pi$ if for all $t = 0, \ldots, T$,
\begin{align}
\label{eq:information_consistent_comm}
    &\mathfrak s_t^l \not \in \left\{\mathfrak s_j^{i} \in {^m\bar{\pi}} \mid 0 \leq j \leq t, \; 0 \leq i \leq n_j \right\} \implies n_t = l.
\end{align}
\end{definition}

Equation \eqref{eq:information_consistent_mob} implies that an agent is not allowed to move until the information path $^m \bar {\pi}$ has reached the initial state $s_0$. Since an information path consists of one mobility edge in each time step, enforcing $n_t = l$ as in \eqref{eq:information_consistent_comm} denies the usage of any communication edge at time $t$ and stage $l$. Thus an agent at state $s_t^l$ is not allowed to transmit information until the information path $^m\bar{\pi}$ has reached that state. We are now ready to formulate the information-consistent version of Problem \ref{prob:p1}.

\begin{problem}
[Information-Consistent Intermittent Connectivity] 
Consider a network $\EuScript{N}$ with R agents, initial conditions $\{ s_0^r \}$, time horizon $T$, and a designated master $m \in \{ 1, \ldots, R\}$. Find a collection of mobility paths $\{ \pi^{r} \}$ such that:
\begin{itemize} 
    \item for each pair of \texttt{src} $i$ and \texttt{snk} $j$ there exists an information path $\bar \pi^{i,j}$ that is valid with respect to $\{ \pi^r \}$,
    \item for each agent $r$ there exists an information path $^m \bar \pi_{r}$ starting in $s_0^m$ that is valid with respect to $\{ \pi^r \}$,
    \item all mobility paths are information-consistent with $^m \bar \pi_{r}$,
    \item all information paths are information-consistent with $^m \bar \pi_{i}$,
\end{itemize}
and such that
\begin{equation}
    \label{eq:p2_obj}
     \mathfrak R - \sum_r \sum_{t=0}^{T-1}  C(t, s_t^r, s_{t+1}^r) - \hspace{-5mm} \sum_{b \in \texttt{src} \cup \{ m\}} \sum_{(t, s, s') \in E_b} \hspace{-3mm}  \widetilde C(t, s, s')
\end{equation}
is maximized, where $\mathfrak R$ is given by \eqref{eq:reward_structure} and $E_b$ by \eqref{eq:edge_set}.
\label{prob:p2}
\end{problem}

Problem \ref{prob:p2} is analogous to Problem \ref{prob:p1} but with the additional requirement that for each agent $r$, there should exist a valid information path $^m\bar\pi_{r}$ that makes the mobility path $\pi^{r}$ and all information paths $\bar\pi^{r,j}$ information-consistent. 

\section{Intermittent Connectivity via Flows}
\label{sec:solution}

In combinatorial optimization, network flows are instrumental for solving many graph optimization problems. Examples include the maximal flow and minimum-cost flow problems \cite{Ahuja1993}. Flow problems generally involve finding a feasible flow that optimizes some cost function. A flow is feasible if it starts in a source state, ends in sink states, and the flow over each edge is less than its capacity. We note that the intermittent connectivity problem has similarities with the multi-commodity flow problem with multiple flows, sources and sinks. Both Problem \ref{prob:p1} and \ref{prob:p2} involve finding multiple information paths with predefined initial states. However, there are two significant differences between a standard flow problem and the intermittent connectivity problems: in the intermittent connectivity problems, the terminal state for each information path is the location of an agent instead of a static state, and edge capacities depend on the movement of agents. 

\subsection{Intermittent Connectivity Problem}
\label{sec:dist_int_sec}

In the following we introduce variables and constraints that form an ILP that solves Problem \ref{prob:p1}.

\subsubsection{Variables}

We start by introducing optimization variables. Since there are multiple information flows we annotate flow variables with a \emph{flow identifier} (or id).

\begin{definition}
Define variables $z_{rst}$, $y_{sk}$, $x_{rijt}$, $f^b_{ijt}$, $\bar{f}^b_{ijt}$ as
\begin{align}
    z_{rst} =  &\begin{cases}
                    1 \text{ if agent $r$ is at state $s$ at time $t$},\\
                    0 \text{ otherwise, }
                \end{cases} 
                \label{var:zrvt}\\
    y_{sk} =  &\begin{cases}
                    1 \text{ if there are $\geq k$ agents in $s$ at time $T$},\\
                    0 \text{ otherwise, }
                \end{cases}
                \label{var:yvk}\\
    x_{rss't} =  &\begin{cases}
                    1 \text{ if agent $r$ uses $(s,s') \in \rightarrow$ at time $t$},\\
                    0 \text{ otherwise, }
                \end{cases}
                \label{var:xrijt}\\
    {f}^b_{ss't} =  &\text{ flow over $(s,s')\in\rightarrow$ with id $b$ at time $t$},
    \label{var:fbijt}\\
    \bar{f}^b_{ss't} = &\text{ flow over $(s,s')\in\rightsquigarrow$ with id $b$ at time $t$}.
    \label{var:fbarbijt}
\end{align}
\end{definition}

The variables $z_{rst}$ and $x_{rss't}$ are binary variables that encode positions and transitions of agents at each time step while $y_{sk}$ counts how many agents that occupy certain states at the final time step. To represent information transmitted over mobility and communication edges we use $f^b_{ss't}$ and $\bar{f}^b_{ss't}$. 

\subsubsection{Dynamics constraints}

For space reasons we omit the sets of ``for all'' quantifiers in the following. That is, instead of $\forall r \in \{ 1, \ldots, R \}$, $\forall s \in S$, and $\forall t \in \{ 0, \ldots, T \}$ we simply write $\forall r$, $\forall s$, and $\forall t$. We can encode the dynamics of the agents in the mobility-communication-network $\EuScript{N}$ using the variables \eqref{var:zrvt} and \eqref{var:xrijt} \cite{Banfi2018}.  The possible transitions at each time can be written as follows.
\begin{align}
    z_{rs(t+1)} &= \sum_{s'\in \mathbf{C}_{\rightarrow}^{-}(s)}  x_{rs'st}, \quad \forall r, s, t, \label{eq:in_dyn}\\
    z_{rst} &= \sum_{s'\in \mathbf{C}_{\rightarrow}^{+}(s)} x_{rss't}, \quad \forall r, s, t. \label{eq:out_dyn}
\end{align}
Equation \eqref{eq:in_dyn} means that an agent transiting to a state $s$ must do so via an incoming mobility edge $(s',s)\in\rightarrow$ and equation \eqref{eq:out_dyn} specifies that an agent at state $s$ must use an outgoing mobility edge $(s,s')\in\rightarrow$.

\subsubsection{Flow constraints}

In order to establish information paths from each \texttt{src} to each \texttt{snk} we look at the net inflow of information with identifier $b$ to a state. To simplify notation, let $F_{st}^b$ denote the net inflow to state $s$ at time $t$ of flow $b$:
\begin{equation}
\label{eq:netflow}
\begin{aligned}
    F_{st}^b = &\sum_{s'\in C_{\rightarrow}^-(s)} f^b_{s's(t-1)} + \sum_{s'\in C_{\rightsquigarrow}^-(s)} \bar{f}^b_{s'st} \\
    & - \sum_{s'\in C_{\rightarrow}^+(s)} f^b_{ss't} - \sum_{s'\in C_{\rightsquigarrow}^+(s)} \bar{f}^b_{ss't}.
\end{aligned}
\end{equation}

We can construct a one-to-many flow from one agent $i\in\texttt{src}$ to all agents $j\in\texttt{snk}$ and thus incorporate all information paths $\underset{j\in\texttt{snk}}{\cup} \bar\pi^{i,j}$ into a single flow. In this case the identifier $b$ for the flow corresponds to the source agent. The resulting constraint can be written as follows.
\begin{equation}
\label{eq:flow_onetomany}
    F_{st}^b = \begin{cases}
        -|\texttt{snk}| z_{bs0}, & \text{if $t = 0$}, \\
        \sum\limits_{r \in \texttt{snk}} z_{rsT}, & \text{if $t = T$}, \\
        0, & \text{otherwise},
    \end{cases} \qquad \forall b \in \texttt{src}, s, t.
\end{equation}
We note that the state including the source agent initially outputs $|\texttt{snk}|$ units of flow and that the net inflow to any state is zero for $0 < t < T$. This ensures that information can not vanish or be created and that all information initially comes from the source agents. The condition for the final time $T$ implies that a state with $k$ agents from \texttt{snk} at time $T$ receives $k$ units of this flow. Therefore each agent in \texttt{snk} is guaranteed to receive one unit of the flow with identifier $b$. As this constraint should hold for all information paths we enforce this constraint for all source agents $b\in\texttt{src}$. Note that in constraint \eqref{eq:flow_onetomany} the terminal state of each agent is free and is determined by $z_{rst}$, as opposed to classical flow problems where it is predefined.

\subsubsection{Bridge constraints}

As previously mentioned, in classical flow problems the capacity of each edge is static and predefined. However, in a mobility-communication network with dynamic agents, the capacity of each edge is dynamic and a function of the position of the agents. For communication edges, information can be transmitted over an edge $(s,s')\in\rightsquigarrow$ at time $t$ only if there is an agent in state $s$ and another agent in state $s'$ at time $t$. This can be written as the big-M constraint \eqref{eq:communication_capacity} where $N\geq R$. For a mobility edge, information can only be transferred over the edge if an agent traverses it. Therefore, the flow over a mobility edge $(s,s')\in\rightarrow$ at time $t$ must be zero unless at least one agent utilizes that edge, which can be written as the big-M constraint \eqref{eq:mobility_capacity}.
\begin{align}
    & \bar{f}^b_{ss't} \leq N \min\left(\sum_r z_{rst}, \sum_r z_{rs't}\right), \quad \forall b, s, s', t, \label{eq:communication_capacity} \\
    & f^b_{ss't} \leq N \sum_r x_{rss't}, \quad \forall b, s, s', t. \label{eq:mobility_capacity}
\end{align}

\subsubsection{Cost function}
\label{sssec:cost_function}

To implement the objective function in Problems \ref{prob:p1} and \ref{prob:p2}, we can penalize state transitions and information transmission costs $C$ and $\bar{C}$ using the variables \eqref{var:zrvt}-\eqref{var:fbarbijt}. The cost of using the mobility (resp. communication) edge $(s,s')$ at time $t$ can be written as $x_{rss't}C(t, s, s')$ (resp. $\bar{f}_{ss't}^b \bar{C}(t, s, s')$). We can formalize the reward of reaching state $s$ with $k$ agents in the final time as $\mathfrak R(s,k)y_{sk}$. We now apply the cost and reward for all time steps, states, edges, and agents and get the following objective function that is linear in the variables introduced above:
\begin{align}
    \label{eq:objective} c_\textrm{obj} &= \sum_{s \in S} \left( \sum_{k = 1}^R y_{sk}\mathfrak R(s,k) - g_1(s) - g_2(s) \right), \\
    g_1(s) &= \sum_{\substack{r\in \{1,\ldots,R \}\\s'\in C^{-}_\rightarrow(s)\\t\in \{0,\ldots,T-1 \}}} x_{rs'st}C(t,s',s), \\   g_2(s) &=  \sum_{\substack{s'\in C^{-}_\rightarrow(s)\\t\in\{1,\ldots,T\}\\b \in \texttt{src}}} \bar{f}_{s'st}^b \bar{C}(t,s',s).
    \label{eq:g2}
\end{align}
Note that for the transition and communication costs, only the incoming edges $C^-$ are used to prevent double-counting of costs. Collecting a reward is restricted by $y_{sk}$ to agents final state, therefore the reward $\mathfrak R$ in Problem \ref{prob:p1} and \ref{prob:p2} is 
$$\mathfrak R = \sum_{s \in S} \sum_{k = 1} ^R y_{sk}\mathfrak R(s,k).$$ 
It follows by inspection that 
\begin{align*}
&\sum_r \sum_{t=0}^{T-1}  C(t, s_t^r, s_{t+1}^r) = \sum_{s \in S} g_1(s), \\ 
&\sum_{b \in \texttt{src} \cup \{ m\}} \sum_{(s, s') \in E_b}  \widetilde C(t, s', s) = \sum_{s \in S} g_2(s).
\end{align*}
Therefore \eqref{eq:objective} is equivalent to the objective function \eqref{eq:p1_obj} in the intermittent connectivity problem. If we change the index set to $b\in \texttt{src} \cup \{ m \}$ in \eqref{eq:g2}, then \eqref{eq:objective} is equivalent to the information-consistent objective function \eqref{eq:p2_obj}.

As the terminal state reward $\mathfrak R(s,k)$ can only be collected if there are $k$ agents in state $s$ at time $T$ we need to constrain $y_{sk}$. Enforcing the following constraint for all $k$:s associated with rewards is equivalent with Definition \ref{var:yvk}.
\begin{align}
    &  y_{sk} \leq \frac{\sum_r z_{rsT}}{k}  \quad \forall s. \label{eq:multi-reward-constraint}
\end{align}
We are now ready to pose a solution to Problem \ref{prob:p1}.
\begin{proposition} 
\label{prop:s1}
Consider the ILP:
\begin{equation}
\label{eq:ilp_prop1}
\begin{aligned}
    \max \quad &\eqref{eq:objective},\\
    \textrm{\normalfont subject to} \quad & \eqref{eq:in_dyn}, \eqref{eq:out_dyn}, \eqref{eq:flow_onetomany}-\eqref{eq:mobility_capacity}, \eqref{eq:multi-reward-constraint}.
\end{aligned}
\end{equation}
A solution to \eqref{eq:ilp_prop1} is also a solution to Problem \ref{prob:p1}.
\end{proposition}

\begin{proof}
    By the argument in Section \ref{sssec:cost_function}, equation \eqref{eq:objective} is equivalent to the cost function in Problem \ref{prob:p1} when \eqref{eq:multi-reward-constraint} holds. By \cite{Banfi2018} constraints \eqref{eq:in_dyn} and \eqref{eq:out_dyn} result in dynamically feasible mobility paths. Under \eqref{eq:flow_onetomany}, for each agent $r \in \texttt{src}$ there exists a flow over the time-extended graph that starts at $s_0^r$ and that has sinks at the locations at time $T$ of the $\texttt{snk}$ agents. The flow decomposition theorem \cite[Theorem 3.5]{Ahuja1993} implies that this flow can be decomposed into information paths, and these information paths are valid since \eqref{eq:communication_capacity} and \eqref{eq:mobility_capacity} prohibit flow over edges that violate the validity requirement.
\end{proof}

\begin{remark}
We can pose an alternative constraint to \eqref{eq:flow_onetomany} that also guarantees existence of the desired information paths. Instead of constructing a one-to-many flow, we can use a many-to-one flow from all agents $i\in\texttt{src}$ to a single agent $j\in\texttt{snk}$. The interpretation of the identifier $b$ for this flow then changes to the single sink agent $j\in\texttt{snk}$. This formulation would yield a different cost associated with the communication edges in the objective in Problems \ref{prob:p1} and \ref{prob:p2} as the sum would be over sinks instead of sources. This many-to-one flow is formalized by the following constraint:
\begin{equation}
\label{eq:flow_manytoone}
    F_{st}^b = \begin{cases}
        -\sum\limits_{r\in \texttt{src}} z_{rs0}, & \text{if $t = 0$}, \\
        |\normalfont\texttt{src}| z_{bsT}, & \text{if $t = T$}, \\
        0, & \text{otherwise},
    \end{cases} \qquad \forall b \in \normalfont\texttt{snk}, s, t.
\end{equation}
Each one of the constraints \eqref{eq:flow_onetomany} and \eqref{eq:flow_manytoone} ensure the existence of an information path from each source $r\in \texttt{src}$ to all sinks $r\in \texttt{snk}$. Constraint \eqref{eq:flow_onetomany} results in a total of $|\texttt{src}| |S| (T+1)$ constraints and constraint \eqref{eq:flow_manytoone} in $|\texttt{snk}| |S| (T+1)$. It may be advantageous to select the one with the lowest number.
\end{remark}

\subsection{Information-Consistent Intermittent Connectivity Problem}
\label{sec:cent_int_sec}

We now seek to extend the result in Proposition \ref{prop:s1} to a solution for Problem \ref{prob:p2}. The difference between Problem \ref{prob:p1} and \ref{prob:p2} is the existence of yet another information path $^m \bar \pi_r$ for each agent $r$ that the mobility path $\pi^r$ and all information paths $\bar \pi^{r,j}$ are information-consistent with. In the following we will refer to the information path $^m \bar \pi_r$ as the master information path for $r$ as it is initialized by the master agent. In the same manner as in Section \ref{sec:dist_int_sec}, we associate the union of the master information paths $\underset{r}{\cup} {^m \bar \pi_r}$ with a flow in the network, naturally called the \emph{master flow}. Denote the net inflow of the master flow into state $s$ at time $t$ by $F_{st}^m$, defined as in \eqref{eq:netflow}.

\subsubsection{Master constraints}

The main difference between a regular flow $F^b$ and the master flow $F^m$ is that the amount of flow $F^b$ and its terminal state are predefined in terms of agent positions, leading to equality constraints in \eqref{eq:flow_onetomany} and \eqref{eq:flow_manytoone}. For the master flow, there is no predefined amount of flow nor fixed terminal state. Therefore, we cover the worst case where master flow is supplied to every state in the network, leading to the initialization of at most $|S|$ units of master flow in the state containing the master agent. As there are no hard constraints on the terminal state for the master flow, we use inequality constraint to allow any terminal state. Let $s_0(r)$ denote the initial state of agent $r$ and let $m$ denote the master agent; the constraint on the master flow can be formalized as follows:
\begin{equation}
    \label{eq:master_flow} 
    F_{st}^m \geq \begin{cases}
         -|S|, & \text{if $s = s_0(m), t = 0$}, \\
        0, & \text{otherwise},
    \end{cases} \qquad \forall s, t. 
\end{equation}
In order to make sure that mobility path $\pi^r$ and all information paths $\{\bar \pi^{r,j}\}$ are information consistent with some master information path $^m \bar \pi_r$ we need to enforce that agent $r$ is \textit{(i)}, static, and \textit{(ii)}, does not transmit information to other agents until its initial position intersects with $^m \bar \pi_r$. This can be ensured with the following constrains that should hold for all $r$ such that $s_0(r) \neq s_0(m)$:
\begin{align}
    \label{eq:master_static}
    z_{rs_0(r)t} \geq 1 - \sum_{\tau=0}^{t-1} F_{s_0(r) \tau}^m,  \quad \forall r, t, \\
    \label{eq:master_comm_bound}
    N \sum_{\tau = 0}^t F^m_{s_0(r)\tau} \geq \sum_{j} \bar{f}^b_{s_0(r)jt}, \quad \forall b, r, t.
\end{align}
These constraints require a net outflow of the master flow to have occurred, as captured by the sum of net flow over time, before mobility (eq. \eqref{eq:master_static}) or communication (eq. \eqref{eq:master_comm_bound}) is allowed. Together with flow balancing \eqref{eq:master_flow} these constraints imply information consistency as defined in Section \ref{ssec:ic_intermitt_conn}.


\begin{proposition} 
\label{prop:s2}
Consider the ILP:
\begin{equation}
\label{eq:ilp_prop2}
\begin{aligned}
    \max \quad &\eqref{eq:objective}\\
    \textrm{\normalfont subject to} \quad & \eqref{eq:in_dyn}, \eqref{eq:out_dyn}, \eqref{eq:flow_onetomany}-\eqref{eq:mobility_capacity}, \eqref{eq:multi-reward-constraint}, \eqref{eq:master_flow}-\eqref{eq:master_comm_bound}
\end{aligned}
\end{equation}
A solution to \eqref{eq:ilp_prop2} is also a solution to Problem \ref{prob:p2}.
\end{proposition}

\begin{proof}
    Problem 2 is a further constrained version of Problem \ref{prob:p1}, thus Proposition \ref{prop:s1} applies for the constraints that are common between the two. As motivated above, constraints \eqref{eq:master_flow}-\eqref{eq:master_comm_bound} in addition ascertain the information consistency requirements in Problem 2.
\end{proof}


\subsection{Extensions}

In order to demonstrate the flexibility of the optimization framework, this section demonstrates multiple extension relevant to multi-agent networks.

\subsubsection{Agent Heterogeneity}

In a multi-agent exploration setting, agents might be equipped with different types of sensors so that only a subset of the agents are capable of exploring and mapping unknown parts of the environment, while other agents perform tasks in the known environment. Therefore it can be useful to only allow a subset of agents to collect rewards by reaching a terminal state. This is can be solved by only summing over agents capable of exploring in constraint \eqref{eq:multi-reward-constraint}. Agents may also have different mobility capabilities (e.g., ground and aerial agents), which can be accommodated by imposing additional constraints that specify that certain mobility edges can only be used by certain agents.

In systems where it is useful to have static agents, such as exploration with a static base station, stationarity must be considered in the planning. This is incorporated by not allowing an agent to leave its initial position using the constraint
\begin{equation}
    z_{rst} = 
    \begin{cases}
        1, & \text{if $s = s_0(r)$}, \\
        0, & \text{otherwise},
    \end{cases}
    \quad \forall t, r \in R_{static},
\end{equation}
where $R_{static}$ is the set of static agents.

\subsubsection{Collision Avoidance}

In multi-agent networks, we can guarantee that two agents $i$ and $j$ can not be in the same position or use facing edges by enforcing the following constraints: 
\begin{align}
    x_{iss't}  \hspace{-1mm} +  \hspace{-1mm} x_{js'st}  \hspace{-1mm} & \leq  \hspace{-1mm} 1 \; \forall t, (s,s') \in \rightarrow, (i,j)\in \hspace{-1mm} \binom{\{1,\ldots,R\}}{2}, \label{eq:trans_avoid}\\
    z_{ist} + z_{jst} & \leq 1 \; \forall t, s, (i,j) \in \binom{\{1,\ldots,R\}}{2}, \label{eq:pos_avoid}
\end{align}
where $\binom{\{1,\ldots,R\}}{2}$ is the set of all possible 2-combinations of agents in the network. In a heterogeneous group of agents, this constraint might only be necessary for a certain subsets of agents and vertices in the network. As an example, an application may require that ground agents avoid other ground agents, that aerial agents avoid other aerial agents, but allow ground and aerial agents to be in the vicinity of each other.

\subsubsection{Awareness-Based Reward}

In applications such as exploration or mapping it might not be enough to only have agents in beneficial positions unless they also have information about its duties, such as exploration a new area or acting as a relay for other agents. It is therefore natural to require that agents have knowledge about the execution plan in order to collect rewards. This is not an issue in Problem \ref{prob:p1} as all agents plan deterministically, but for Problem \ref{prob:p2} we have to restrict the ability to receive the terminal state reward $\mathfrak R(v,k)$ to agents with knowledge about the solution. We can accomplish this with the following constraint:
\begin{equation}
    \label{eq:master_reward}
    y_{sk} \leq \sum_{\tau = 0}^T F_{s\tau}^m  \quad \forall k, s \neq s_0(m). 
\end{equation}
Constraint \eqref{eq:master_reward} together with \eqref{eq:master_static} ensures that an agent that collects a reward has been provided with master information. 

\subsubsection{Multiple Master Agents}
\label{sssec:multiple_master}

In some settings the master plan may be available to multiple agents at the initial time step. This would for example be the case if a subset of agents solve Problem \ref{prob:p2} deterministically and find the same unique solution, or in a tunnel exploration scenario where planning is done outside the tunnel but the tunnel has multiple openings where agents can enter. It is straightforward to extend Problem \ref{prob:p2} and its solution to the case with multiple masters by substituting the condition $s_0(r) \not = s_0(m)$ for $s_0(r) \not \in \{s_r(0)|r\in m\}$ in constraints \eqref{eq:master_flow}-\eqref{eq:master_comm_bound} and \eqref{eq:master_reward}. This corresponds to initializing $|S|$ units of master flow in the initial state of each master agent. Even with this extension only a single master flow is used since the master information is the same irrespective of master agent.

\begin{figure}[t]
    \begin{center}
    \resizebox{\columnwidth}{!}{
        \begin{tikzpicture}[auto, node distance=0.6cm,>=latex']
        \node [block, fill = gray!20, name=planner, text width=4cm,align=center]{Update graph \\Plan 'Pre-Exploration' \\ Plan 'Post-Exploration'};
        \node [block, fill = gray!20, right=of planner, text width=3cm,align=center] (exe1) {Execute \\ 'Pre-Exploration'};
        \node [block, fill = gray!20, right= of exe1, text width=2cm,align=center, pin={[pinstyle]above:Unknown environment}] (exp) {Explore \\ frontiers};
        \node [block, fill = gray!20, right=of exp, text width=3cm,align=center] (exe2) {Execute \\ 'Post-Exploration'};
        \draw [-latex] (planner) -- node {} (exe1);
        \draw [-latex] (exe1) -- node {} (exp);
        \draw [-latex] (exp) -- node {} (exe2);
        \draw [-latex] (exe2) -- ++ (0,-1) -| node [pos=0] {Explored graph} (planner);
        \draw [-latex] (planner) -- ++ (2.4,0)-- ++ (0,1)-- ++ (-1.7,0) -| node [above, text width=3cm] {'Pre-Exploration' \\ 'Post-Exploration'} (planner);
        \end{tikzpicture}
    }
    \end{center}
    \vspace{-3mm}
    \caption{Block diagram illustrating the use the pre-exploration and post-exploration plans. The final agent positions in the post-exploration problem become the initial positions in the subsequent pre-exploration problem, as indicated by the self-loop in the planning block.}
    \label{fig:blockdia_plan}
    \vspace{-3mm}
\end{figure}
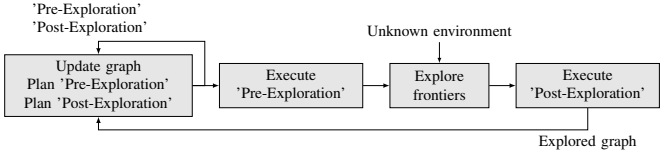

\section{Exploration and Clustering for Scalability}

We now present a way to leverage the framework developed in Section \ref{sec:solution} to explore an unknown environment using a team of mobile agents. We study a case where there is a static base station and a heterogeneous group of agents that should explore a communication-constrained tunnel-like environment. Although the flow approach is capable of solving problems of moderate size, large networks are still challenging since solving ILPs is NP-hard. To mitigate this issue we propose a clustering method that separates a large exploration problem into smaller problems that can be solved independently.

\subsection{Exploration via Intermittent Connectivity}
\label{sec:implementation_section}

The main goal in an exploration scenario is to get agents to \emph{frontiers}, which are states on the boundary between known and unknown space. However, it is also important to share information about the progress between agents and potentially with an external operator. For this reason we require that a static \emph{base station} is periodically updated with the progress.

Since we do not assume continuous connectivity it is necessary to synthesize a plan that gets agents to frontiers, and also a plan that distributes the new information to the base station when frontier exploration is finished. We call these two planning problems \emph{pre-exploration} and \emph{post-exploration}. Fig.  \ref{fig:blockdia_plan} shows how the two problems are solved and executed in a repeating manner to explore an unknown environment. 

In pre-exploration we use Proposition \ref{prop:s2} with a single designated master that computes a global plan. The master could be the base station, or it could be a mobile agent that has sufficient computational capabilities for solving the planning problem. Since the only information that needs to be shared is the plan itself, there are no sources or sinks in this problem instance, i.e., $\texttt{src} = \texttt{snk} = \emptyset$. To promote exploration we assign high rewards to frontiers states. In addition, we include a heuristic reward based on the betweenness-centrality measure of a state \cite{Brandes2001}, which implies that agents get a higher reward by finishing in a state that is more central in the graph. This is useful for efficient data distribution in the subsequent post-exploration step. Executing the resulting plan sends agents to frontier states from which a suitable exploration behavior can be executed. 

For the post-exploration problem we again solve an instance of Problem \ref{prob:p2}, this time with information sources \texttt{src} set to the agents that performed exploration (and thus have new information), and the information sinks $\texttt{src}$ set to the base station and the master agent. It is possible that agents that were idle during pre-exploration are useful in post-exploration. For this purpose, in the post-exploration problem each agent that received the pre-exploration plan is considered a master agent in the post-exploration instance of Problem \ref{prob:p2} (leveraging the multi-master extension discussed in Section \ref{sssec:multiple_master}).

When the master is a static base station it is necessary that it can communicate to a mobile agent for a plan to be distributable. To ensure this property we add an additional constraint to the post-exploration problem that requires at least one dynamic agent to be in communication range with the static master. If we denote the set of dynamic agents by $R_{\textrm{dyn}}$ and the set of states connected to the master, possibly via other static relay agents, by $^m S_{\textrm{static}}$, we add the following constraint to the post-exploration problem to guarantee feasible information distribution in the subsequent pre-exploration problem.
\begin{equation}
    \sum_{r \in R_{\textrm{dyn}}}\sum_{s\in ^m S_{\textrm{static}}} z_{rsT} \geq 1.
\end{equation}

\begin{figure}
    \centering
    \begin{tikzpicture}[scale=0.8]
    	\path[draw, fill=red, fill opacity=0.2, use Hobby shortcut,closed=true] (0,0) .. (.5,1) .. (1,2) .. (.3,3) .. (-1,1) .. (-1,.5);

        \node at (-0.75, 0.4) {$S_1$};

    	\path[draw, fill=yellow, fill opacity=0.2, use Hobby shortcut,closed=true] (1.3, 2) .. (2.5, 1.6) .. (2, 3) .. (1.5, 3.4);

        \node at (1.4, 3.1) {$S_2$};

    	\path[draw, fill=blue, fill opacity=0.2, use Hobby shortcut,closed=true] (1, 1.2) .. (2.5, 0.2) .. (2.3, -0.2);

        \node at (0.95, 0.3) {$S_3$};

    	\path[draw, fill=green!50!black, fill opacity=0.2, use Hobby shortcut,closed=true] (-1.2, 2) .. (-2, 1) .. (-2, 2);

        \node at (-2, 1.3) {$S_4$};

    	\node[draw=black, circle, fill=black] (master) at (-0.3, 0.15) {};
    	\node[draw=black, inner sep=2, circle, fill=red] at (0.2, 1) {};
    	\node[draw=black, inner sep=2, circle, fill=red] at (-0.4, 2) {};
    	\node[draw=black, inner sep=2, circle] at (0.3, 1.5) {};
    	\node[draw=black, inner sep=2, circle] at (0.2, 2.5) {};
    	\node[draw=black, inner sep=2, circle] at (-0.5, 0.7) {};
    
    	\node[draw=black, circle, fill=green!50!black] (subm1) at (-1.3, 1.5) {};
    	\node[draw=black, inner sep=2, circle, fill=green!50!black] at (-1.9, 1.6) {};
    	\node[draw=black, inner sep=2, circle] at (-1.7, 1.1) {};
    
    	\node[draw=black, circle, fill=yellow] (subm2) at (1.3, 2.6) {};
    	\node[draw=black, inner sep=2, circle, fill=yellow] at (2.3, 2) {};
    	\node[draw=black, inner sep=2, circle] at (2, 2.3) {};
    
    	\node[draw=black, circle, fill=blue] (subm3) at (1.65, 1.2) {};
    	\node[draw=black, inner sep=2, circle, fill=blue] (n31) at (1.9, 0.5) {};
    	\node[draw=black, inner sep=2, circle, fill=blue] (n32) at (1.4, 0.2) {};
    	\node[draw=black, inner sep=2, circle] (n33) at (1.7, 0) {};
    	\node[draw=black, inner sep=2, circle] (n34) at (1.2, 0.8) {};
    
    	\draw[latex-, thick, dashed] (subm1) -- ++ (0.5, -0.2) node[anchor=west,  line width=0.5, draw=black, solid, circle, inner sep=2] {};
    
    	\draw[latex-, thick, dashed] (subm2) -- ++ (-0.5, -0.5) node[anchor=north east, line width=0.5,  draw=black, solid, circle, inner sep=2] {};
    
    	\draw[latex-, thick, dashed] (subm3) -- ++ (0.1, 0.5) node[anchor=south,  line width=0.5, draw=black, solid, circle, inner sep=2] {};
    \end{tikzpicture}
    \vspace{-7mm}
    \caption{Illustration of clustering for exploration. States are depicted with circles, and occupied states are filled. The master agent is marked with a black circle. The master cluster $S_1$ has two child clusters $S_2$ and $S_4$, and $S_3$ is the child of $S_2$. Each child cluster has a designated submaster agent (shown as a larger circle) whose state is connected to a state in the parent cluster via a communication edge.
    }
    \label{fig:clustering}
    \vspace{-1mm}
\end{figure}
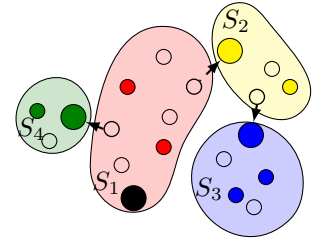

\subsection{Clustering for Improved Scalability}
\label{sec:clustering}

We next exploit the inherent structure of the exploration setting to improve scalability. We remark that agents that are far away from each other in the network are unlikely to directly interact, and therefore aim to divide the network into multiple sub-networks, or \emph{clusters}, and constrain the movement of each agent to a single cluster. We exploit the hierarchical structure of the problem to design an information sharing scheme that passes information between clusters.

\begin{definition}
\label{def:clustering}
    A \emph{clustering} of a mobility-communication network $\EuScript N = (S, \rightarrow, \rightsquigarrow)$ with agent positions $\{s_0^r\}$ and master $m$ is a partition of $S$ into connected (via mobility edges) subsets $S_1, \ldots, S_k$ such that the following holds:
    \begin{itemize}
         \item $s_0^m \in S_1$ and we call $S_1$ the \emph{master cluster},
         \item Each cluster $S_2, \ldots, S_k$ has a designated \emph{parent cluster},
         \item Each cluster $S_i$ except the master cluster has a \emph{submaster}, which is an agent $r$ s.t. $s_0^r \in S_i$ s.t. there is a communication edge from a state in the parent cluster to $s_0^r$.
     \end{itemize} 
\end{definition}

An illustration of the clustering concept in Definition \ref{def:clustering} is shown in Fig. \ref{fig:clustering}. To arrive at a clustering that satisfies these requirements we start by grouping agents together via \emph{spectral clustering} \cite{Shi2000},\cite{Bonaccorso:2017}, which takes a similarity matrix $A$ and desired number of clusters $K$ as inputs and categorizes the agents into clusters by performing $K$-means clustering based on the spectrum of $A$. The similarity matrix $A$ is constructed by defining the similarity $A_{i,j}$ of two agents $i$ and $j$ as the inverse of the shortest weighted distance between $i$ and $j$ using only mobility edges and the transition costs as weights.

Once agents have been grouped together into \emph{agent clusters}, the next step is to find corresponding subgraphs, or \emph{state clusters}. Naturally, if agent $r$ belongs to agent cluster $k$, then $s_0(r)$ should belong to the $k$:th state cluster. Algorithm \ref{alg:clustrings_algorithm} starts by assigning these initial states to clusters, and then iteratively expands the state clusters in a way that guarantees that the conditions in Definition \ref{def:clustering} are satisfied. It starts with only the master cluster being \emph{active}, and at each iteration finds a free (i.e., not assigned to a cluster) state $s$ that is at the minimal distance to any agent in an active cluster $S$, and assigns $s$ to $S$. If $s$ has a communication edge to an agent $r$ in an inactive cluster $S'$, then $S'$ is activated as a child of $S$ and $r$ becomes the submaster of $S'$. In case the algorithm results in a non-connected state cluster, the corresponding agent cluster is split and the algorithm restarted.

\subsection{Exploration Planning on a Clustering}

We next discuss how to plan for exploration as in Fig. \ref{fig:blockdia_plan} by solving intermittent connectivity problems locally in each cluster. For the pre-exploration plan we solve an instance of Problem \ref{prob:p2} for each cluster in a reverse breath-first manner: we first solve for clusters with no children (leaf clusters), proceed by solving for clusters that only have leaf clusters as children, and so on, until the master cluster is reached.

Consider the problem instance for a cluster $S_k$. The submaster of $S_k$ is considered the master agent in the problem instance, and submasters of child clusters of $S_k$ are included as static agents so that information can be transferred to from $S_k$ to its child clusters. The reward for passing information to a submaster is set to the optimal value of the optimization problem in the corresponding child cluster, so that it becomes beneficial to pass the plan to high-value child clusters. If a cluster $S_k$ does not receive plan information from its parent, agents in $S_k$ become inactive and not part of the overall plan.

For the post-exploration plan we require that every cluster that received the pre-exploration plan should send information back to the submaster. Both information from frontiers, as well as information from child clusters, should be passed to the parent cluster. To achieve this we solve an instance of Problem \ref{prob:p2} where each agent that received the pre-exploration plan is considered a master, \texttt{src} contains all agents that reached frontiers and all submasters of child clusters, and \texttt{src} contains the cluster submaster. This forces every agent that may have new information---either from a frontier or from a child cluster---to transmit that information to the cluster submaster. Since this is repeated in every cluster new information from frontiers eventually reaches the global master.

\SetAlgoLined
\LinesNumbered
\begin{algorithm}[t]
    \caption{Agent clustering to state clustering.}
    \label{alg:clustrings_algorithm}
    \KwResult{\begin{itemize}
        \item clusters: map from id to set of vertices
        \end{itemize}}
    \KwIn{\begin{itemize}
        \item $\EuScript{N}$: mobility-communication network
        \item agent\_clusters: map from cluster id to agent cluster
        \item activate($\EuScript{N}, s$): id:s of clusters that $s$ can activate
        \item closest\_free\_state($\EuScript{N}$, clusters): state that is at minimal distance to any agent in a cluster id$\in \text{clusters}$
        \end{itemize}}
    $\text{active\_clusters} = \{\text{master\_cluster}\}$\;
    $\text{clusters[id]} = \{ s_0(r) : r \in \text{agent\_clusters[id]} \}$\;
    $\text{free\_states} = S \setminus \left( \bigcup_k \text{clusters[k]} \right)$\;    
    \While{\normalfont free\_states not empty}{
        $s, \text{id} = \text{closest\_free\_state}(\EuScript{N}, \text{active\_clusters})$\\
        $\text{clusters[id].add}(s)$\\
        $\text{free\_states.remove}(s)$\\
        \For{$\text{new\_id} \in \text{\normalfont activate}(\EuScript{N}, s) \setminus \text{active\_clusters}$}{
            $\text{active\_clusters.add}(\text{new\_id})$\;
        }
    }
\end{algorithm}

\begin{figure*}
    \centering
    \footnotesize
    \begin{tikzpicture}
        \node at (0\textwidth,0) {\includegraphics[width=0.19\textwidth,trim={3cm 3cm 3cm 4cm},clip]{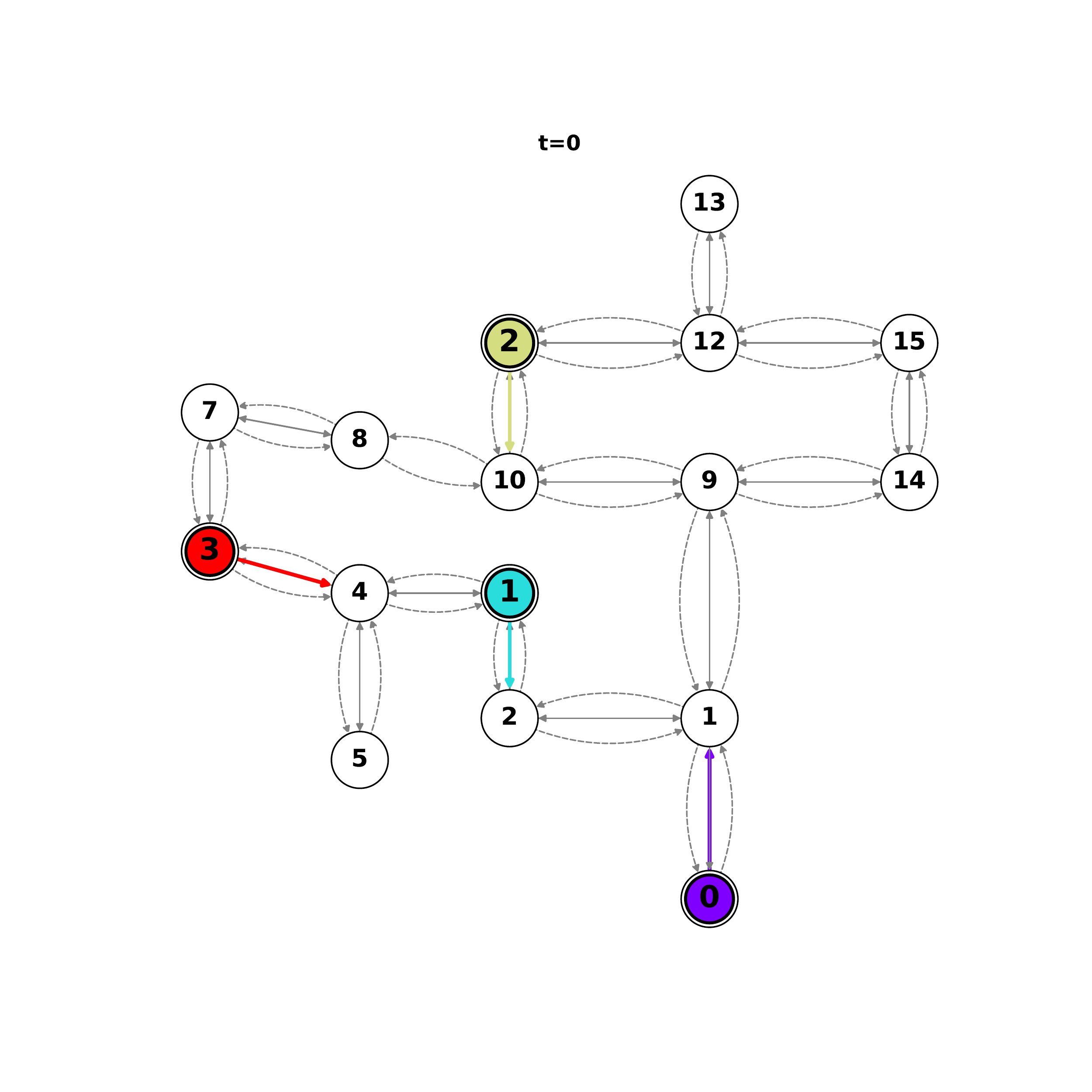}};
        \node at (-0.03\textwidth, -1.3) {$t=0$};

        \node at (0.2\textwidth, 0) {\includegraphics[width=0.19\textwidth,trim={3cm 3cm 3cm 4cm},clip]{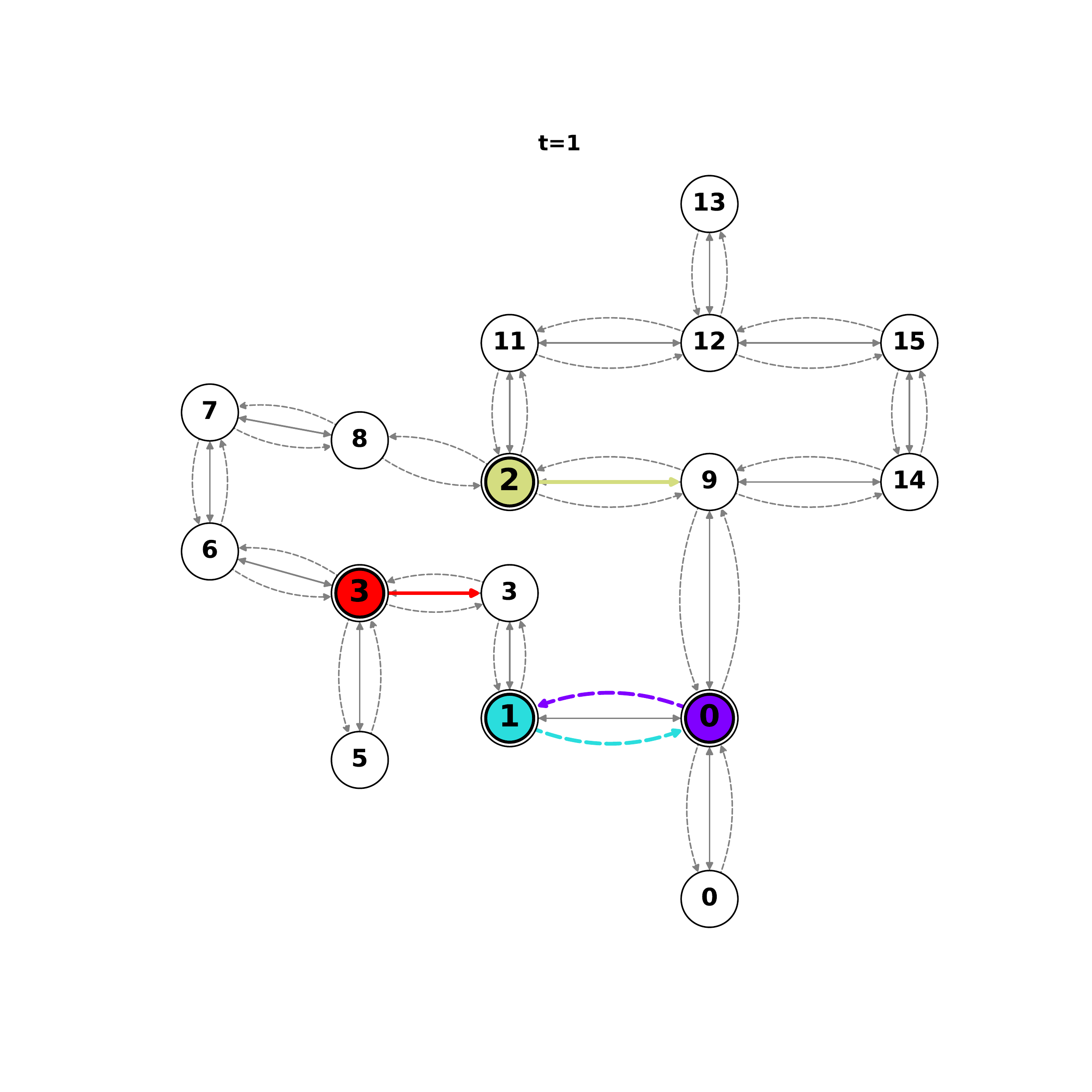}};
        \node at (0.17\textwidth, -1.3) {$t=1$};

        \node at ([yshift=-0.2\textwidth]0\textwidth, 0) {\includegraphics[width=0.19\textwidth,trim={3cm 3cm 3cm 4cm},clip]{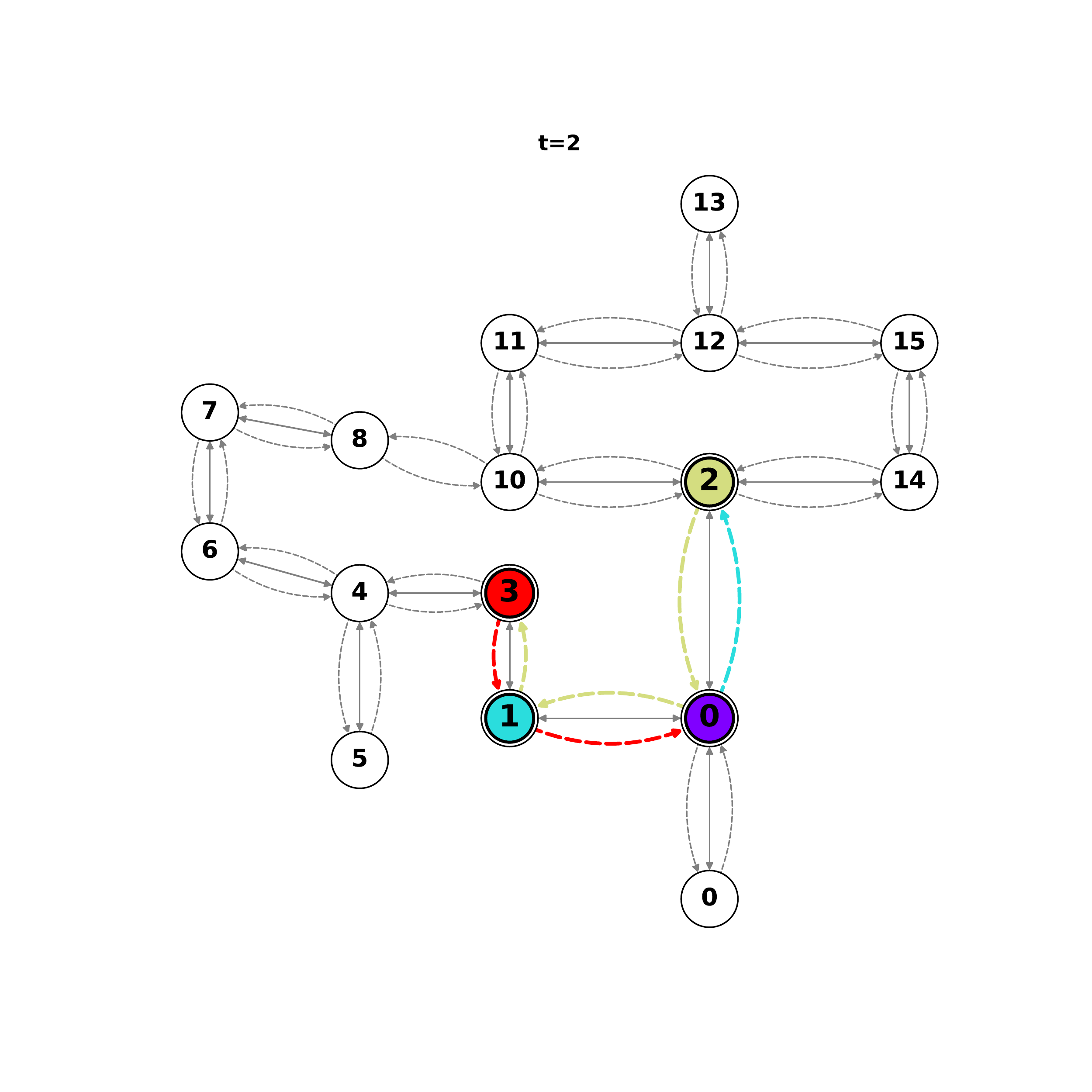}};
        \node at ([yshift=-0.2\textwidth]-0.03\textwidth, -1.3) {$t=2$};

        \node at (0.1\textwidth, -0.3\textwidth) {(a)};

        \draw[dashed] (0.3\textwidth, 1.8) -- (0.3\textwidth, -5.4);

        \node at (0.4\textwidth, 0) {\includegraphics[width=0.19\textwidth,trim={3cm 3cm 3cm 4cm},clip]{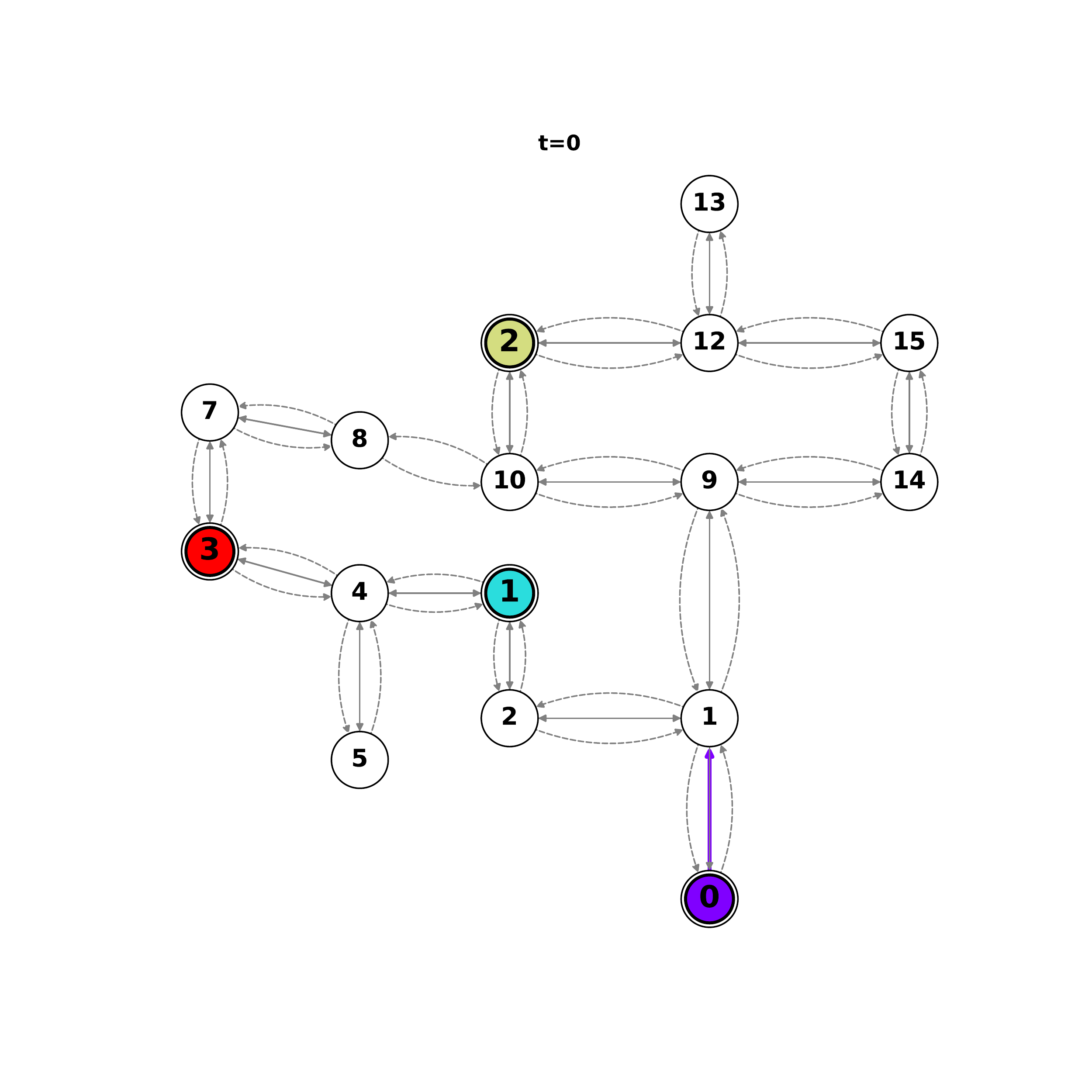}};
        \node at (0.37\textwidth, -1.3) {$t=0$};

        \node at (0.6\textwidth, 0) {\includegraphics[width=0.19\textwidth,trim={3cm 3cm 3cm 4cm},clip]{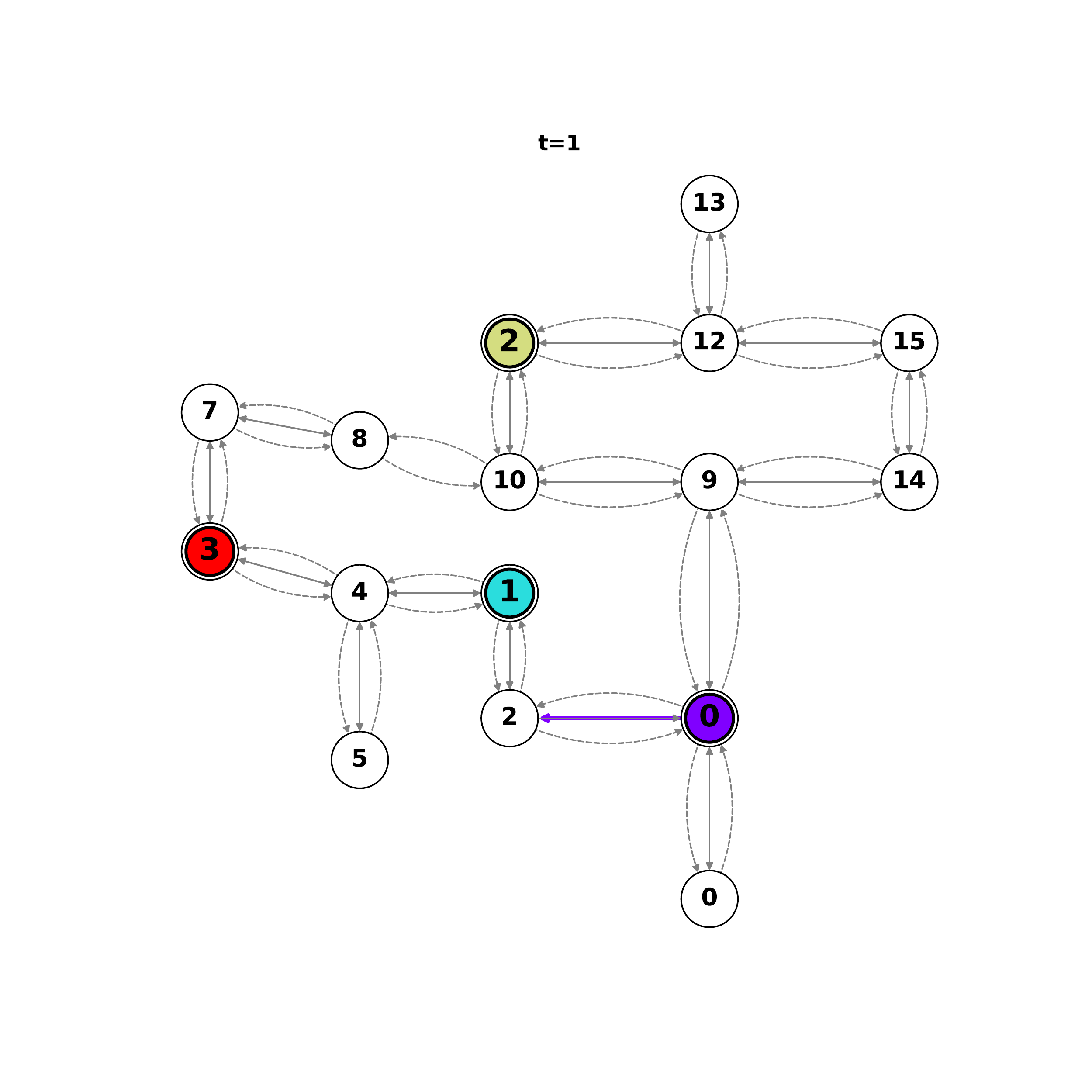}};
        \node at (0.57\textwidth, -1.3) {$t=1$};

        \node at (0.8\textwidth, 0) {\includegraphics[width=0.19\textwidth,trim={3cm 3cm 3cm 4cm},clip]{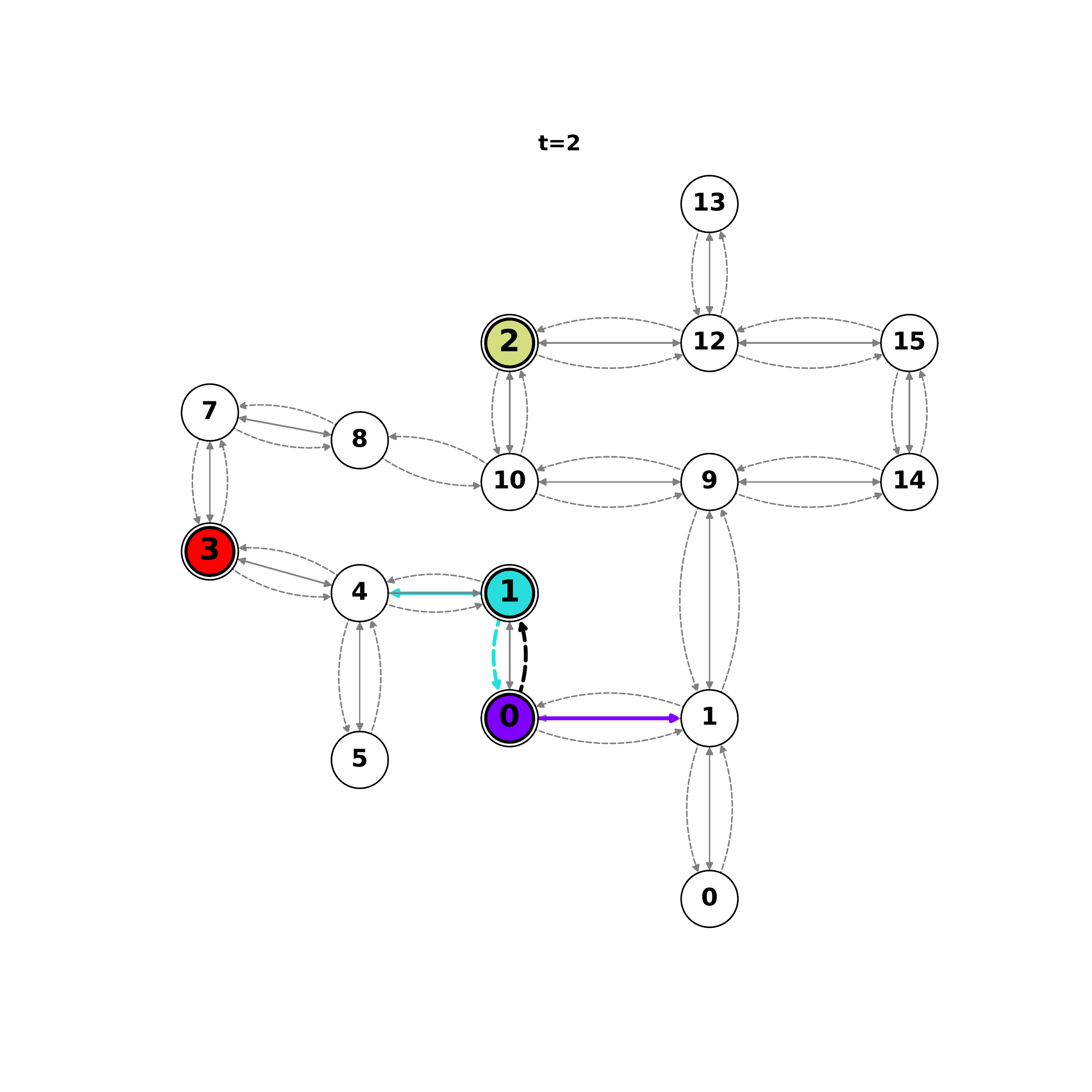}};
        \node at (0.77\textwidth, -1.3) {$t=2$};

        \node at ([yshift=-0.2\textwidth]0.4\textwidth, 0) {\includegraphics[width=0.19\textwidth,trim={3cm 3cm 3cm 4cm},clip]{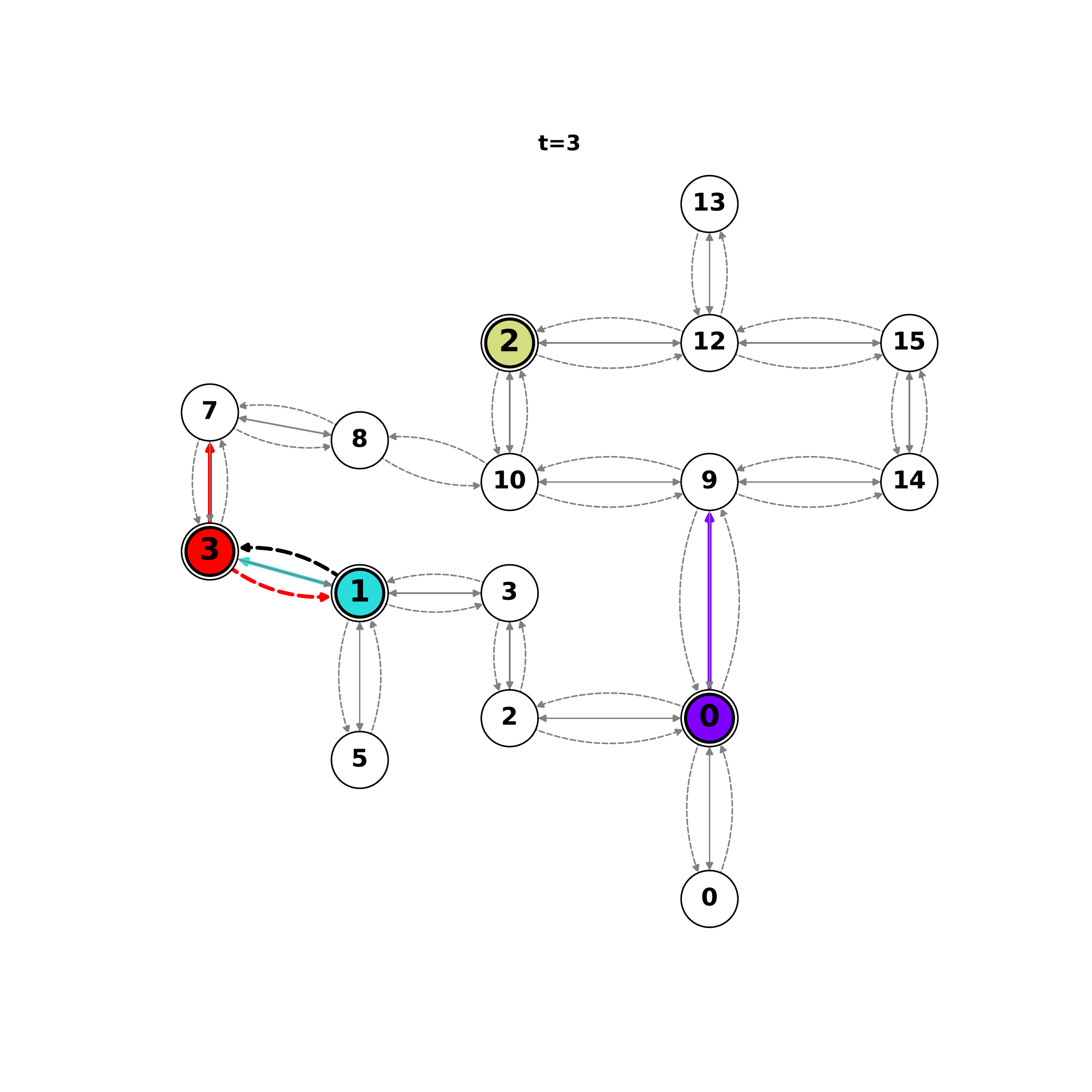}};
        \node at ([yshift=-0.2\textwidth]0.37\textwidth, -1.3) {$t=3$};

        \node at ([yshift=-0.2\textwidth]0.6\textwidth, 0) {\includegraphics[width=0.19\textwidth,trim={3cm 3cm 3cm 4cm},clip]{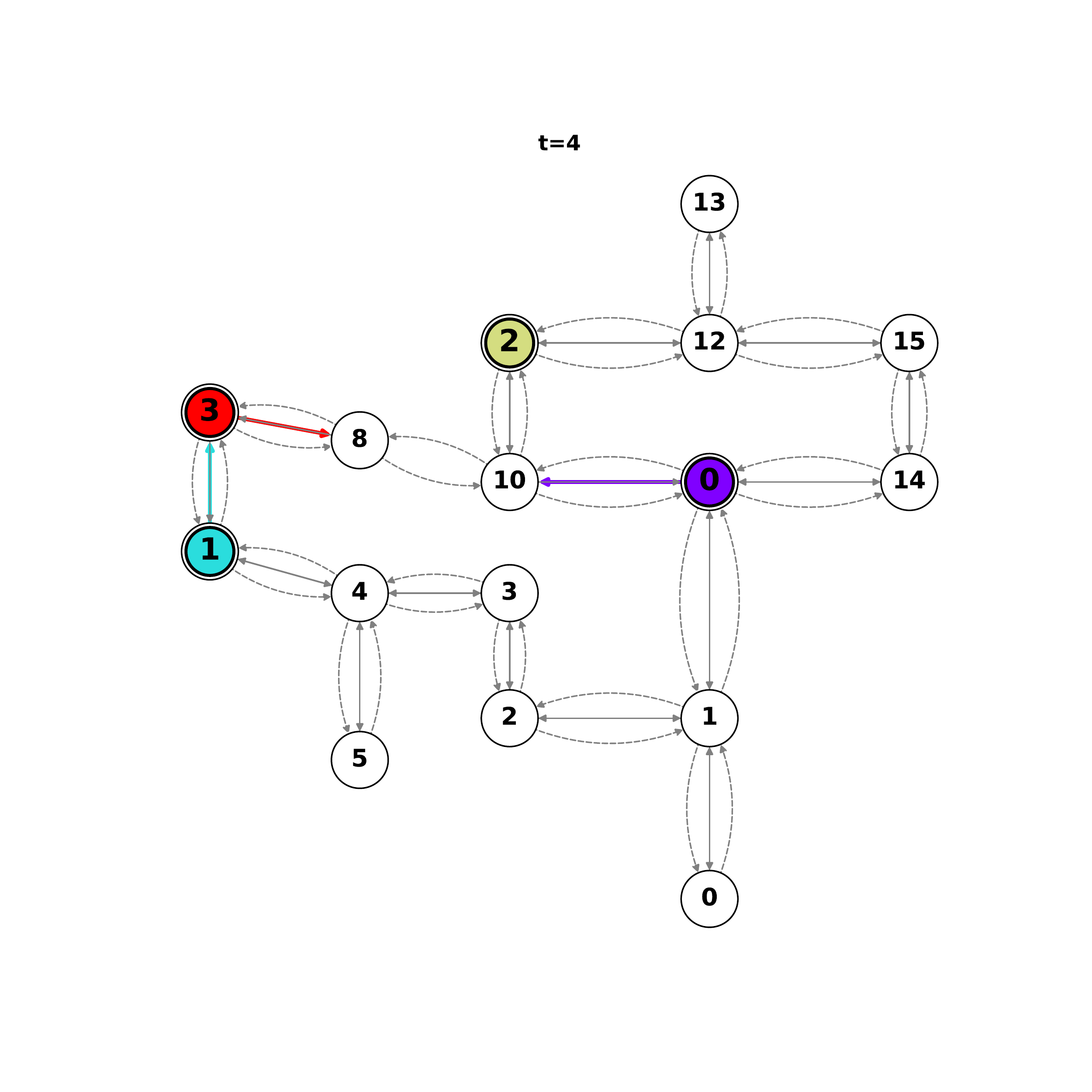}};
        \node at ([yshift=-0.2\textwidth]0.57\textwidth, -1.3) {$t=4$};

        \node at ([yshift=-0.2\textwidth]0.8\textwidth, 0) {\includegraphics[width=0.19\textwidth,trim={3cm 3cm 3cm 4cm},clip]{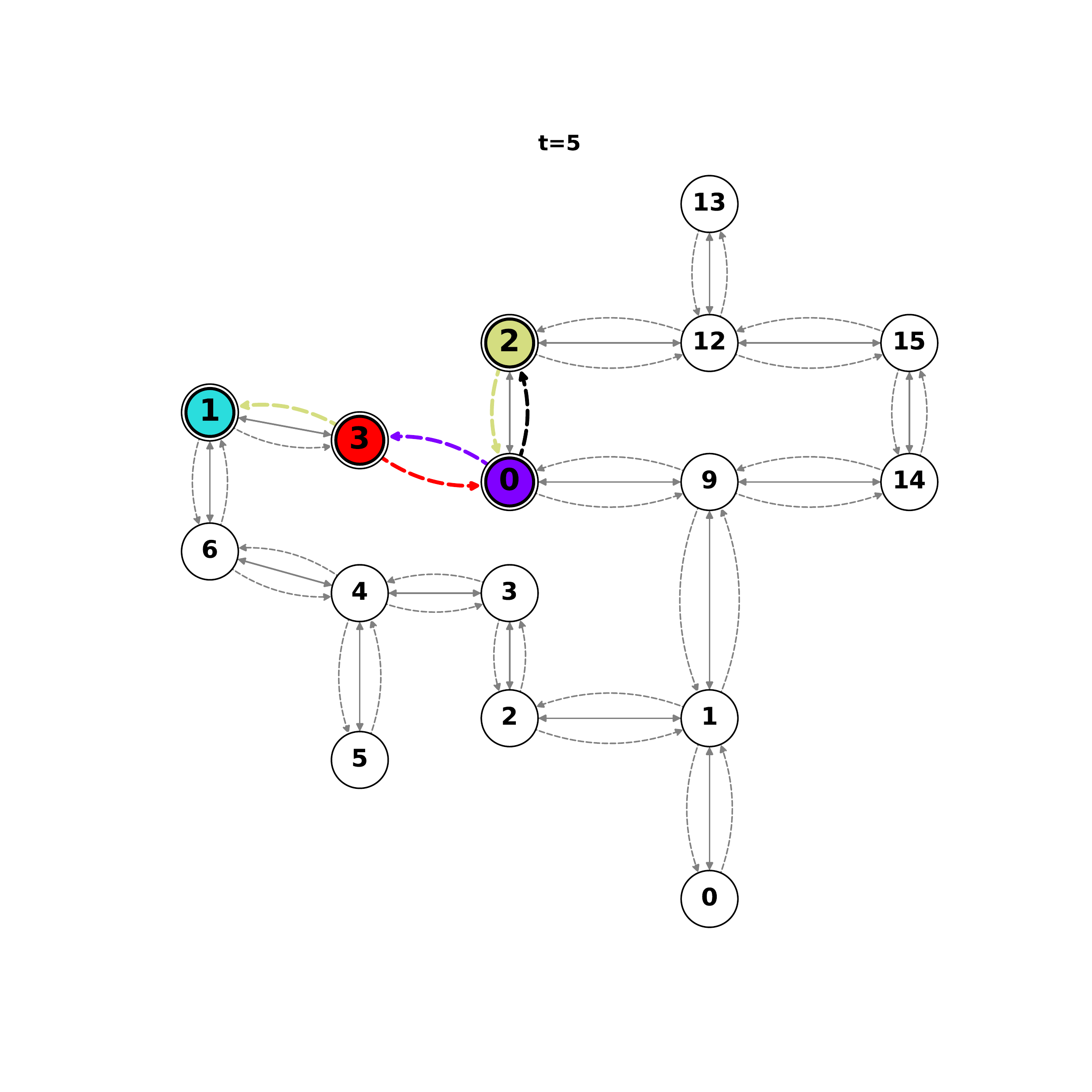}};
        \node at ([yshift=-0.2\textwidth]0.77\textwidth, -1.3) {$t=5$};
        
        \node at (0.6\textwidth, -0.3\textwidth) {(b)};
        
    \end{tikzpicture}    
    \vspace{-5mm}
    \caption{Simulation of intermittent connectivity with (a) and without (b) master constraints. All agents are required to share information with each other.}
    \label{fig:without_master}
    \vspace{-2mm}
\end{figure*}
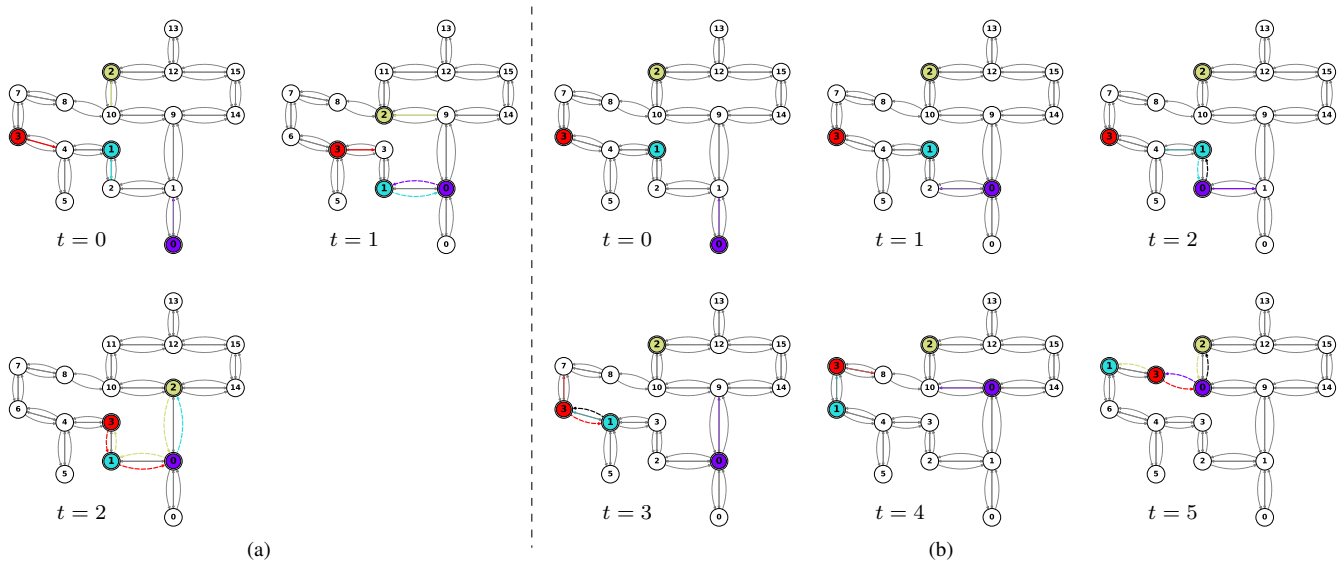

\section{Results}
\label{sec:results}

We next show how the results in this paper can be used to plan for exploration. We start with a numerical example that exhibits the difference between intermittent connectivity, and information-consistent intermittent connecivity, and then show numerical results of the improved scalability of our flow-based constraints for information sharing. Finally, we show how the clustering method can be used to plan for exploration in very large environments. All the results in this section have been computed with our Python toolbox COPS\footnote{COPS is available at \url{https://github.com/FilipKlaesson/cops}.}, using Gurobi \cite{gurobi} as the underlying ILP solver. The computations are performed on a laptop with an Intel Core i7-8850H processor. 


\subsection{Information-Consistency in Intermittent Connectivity}
\label{sec:simulations}

We first demonstrate solutions to the intermittent connectivity problems synthesized by solving the ILPs in Proposition \ref{prop:s1} and \ref{prop:s2}. Fig. \ref{fig:without_master}a shows a solution of Problem \ref{prob:p1} where all four agents are required to share information with each other, i.e. $\texttt{src} = \texttt{snk} = {0,\ldots,3}$. As can be seen from the figures, all agents are moving and sharing information as expected and are therefore the problem only requires $T=2$. However, since all four agents move in the first time step this plan is contingent upon all agents knowing the plan at time $t=0$.

We next solve Problem \ref{prob:p2} to demonstrate what information-consistency implies for the same problem. The purple agent 0 is the master agent and again, we require all 4 agents to share information with each other. The resulting plan is illustrated in Fig. \ref{fig:without_master}b. Naturally, the additional information-consistency requirement implies that more time steps are needed to perform the information sharing task. We see that agents do not perform any tasks before master information (black communication edges) is received, which fulfills the requirement of information consistency.

\begin{figure}[!b]
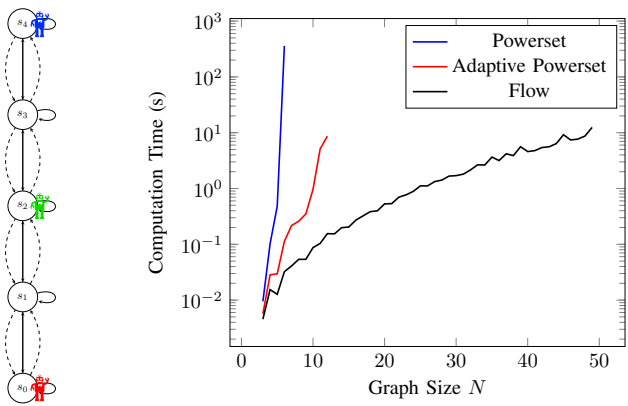

    \vspace{-2mm}
    \centering
    \resizebox{!}{0.6\columnwidth}{%
        \begin{tikzpicture}[>=stealth,node distance=2.7cm,auto]
        \tikzstyle{agent}=[circle,thick,draw=blue!75,fill=blue!20,minimum size=8mm]

        \node[state]                    (s0)                    {$s_0$};
        \node[state]                    (s1) [above of = s0]      {$s_1$};
        \node[state]                    (s2) [above of = s1]       {$s_2$};
        \node[state]                    (s3) [above of = s2]      {$s_3$};
        \node[state]                    (s4) [above of = s3]      {$s_4$};
                 
        \draw[dashed,->] (s0) to [out=60,in=-60] (s1);
        \draw[dashed,->] (s1) to [out=60,in=-60] (s2);
        \draw[dashed,->] (s2) to [out=60,in=-60] (s3);
        \draw[dashed,->] (s3) to [out=60,in=-60] (s4);
        \draw[dashed,<-] (s0) to [out=120,in=-120] (s1);
        \draw[dashed,<-] (s1) to [out=120,in=-120] (s2);
        \draw[dashed,<-] (s2) to [out=120,in=-120] (s3);
        \draw[dashed,<-] (s3) to [out=120,in=-120] (s4);
        
        \path[->]
            (s0) edge   node {}     (s1)
            (s1) edge   node {}     (s0)
            (s1) edge   node {}     (s2)
            (s2) edge   node {}     (s1)
            (s2) edge   node {}     (s3)
            (s3) edge   node {}     (s2)
            (s3) edge   node {}     (s4)
            (s4) edge   node {}     (s3)
            (s0) edge [loop right] node {} (s0)
            (s1) edge [loop right] node {} (s1)
            (s2) edge [loop right] node {} (s2)
            (s3) edge [loop right] node {} (s3)
            (s4) edge [loop right] node {} (s4);
            
        \node at ([xshift=5mm]s0) {\includegraphics[width=0.8cm]{figures/red.png}};
        \node at ([xshift=5mm]s2){\includegraphics[width=0.8cm]{figures/green.png}};
        \node at ([xshift=5mm]s4){\includegraphics[width=0.8cm]{figures/blue.png}};

        \end{tikzpicture}
    }
    ~ \hspace{5mm}
    \resizebox{!}{0.6\columnwidth}{%
        \begin{tikzpicture}
        \begin{semilogyaxis}[xlabel={Graph Size $N$}, ylabel={Computation Time (s)}]
        \addplot[color=blue, thick] table [x=x, y=y, col sep=comma] {data/powerset.csv};
        \addlegendentry{Powerset}
        \addplot[color=red, thick] table [x=x, y=y, col sep=comma] {data/adaptive.csv};
        \addlegendentry{Adaptive Powerset}
        \addplot[color=black, thick] table [x=x, y=y, col sep=comma] {data/flow.csv};
        \addlegendentry{Flow}
        \end{semilogyaxis}
        \end{tikzpicture}
    }
    \vspace{-2mm}
\caption{Computation time of the intermittent connectivity problem as a function of graph size N with agents in state $0$, $\left \lceil \frac{N}{2} \right \rceil$ and $N-1$, and all agents are required to share information with each other.}
\label{fig:scalability}
\end{figure}

\subsection{Scalability}
\label{sec:scalability}

To demonstrate the improved scalability of flow information sharing constraints compared to previous work we compute optimal solutions for a linear graph of length $N$. In these experiments there are three agents positioned in $0$, $\left \lceil \frac{N}{2} \right \rceil$ and $N-1$, with the constraint that all agents are required to share information with each other. We compare the performance with two other methods. In \cite{Banfi2018}, a configuration with recurrent connectivity constraint is planned by looking at the powerset $2^S$ of the vertices in the network and requiring that if an agent is in a set $A \in 2^S$, then the solution has to use at least one communication edge into $A$. We extended the powerset idea to plan for intermittent connectivity over time by including dynamics and enforce constraints as in \cite{Banfi2018} on the time-extended graph. The powerset of the time extended graph scales poorly: there is a copy of the original graph for each time step and hence $2^{(T+1) |S|}$ different subsets.

We also consider an adaptive version of the powerset method that first solves the optimization problem without any communication constraints, and, in case a communication constraint is violated, adds the corresponding subset constraint to the problem and solves it again. Thus, constraints corresponding to additional subsets are added until a feasible solution is found or the problem becomes infeasible.

Computational times for solving problems on a linear graph are presented in Fig. \ref{fig:scalability} for the three methods. The adaptive powerset method outperforms the original powerset method, however both are unable to solve the problem for moderately sized graphs which limits their practical applicability. The flow method scales better and is practical for graphs of moderate size, which is expected since the number of constraints is much lower in the flow formulation: the number scales linearly with the number of edges in the network, instead of combinatorially with graph size. The relationship between the flow approach and the powerset approach is similar to that between the max-flow and min-cut problems. Although the two are dual \cite{Ahuja1993}, flows are parameterized by a linear number of variables, whereas there are combinatorially many cuts.

\begin{figure*}[t]
    \hspace*{-0.8cm}
    \begin{tikzpicture}
        \node at (0\textwidth, 0) {\includegraphics[width=0.3\textwidth,trim={3cm 3cm 3cm 4.2cm},clip]{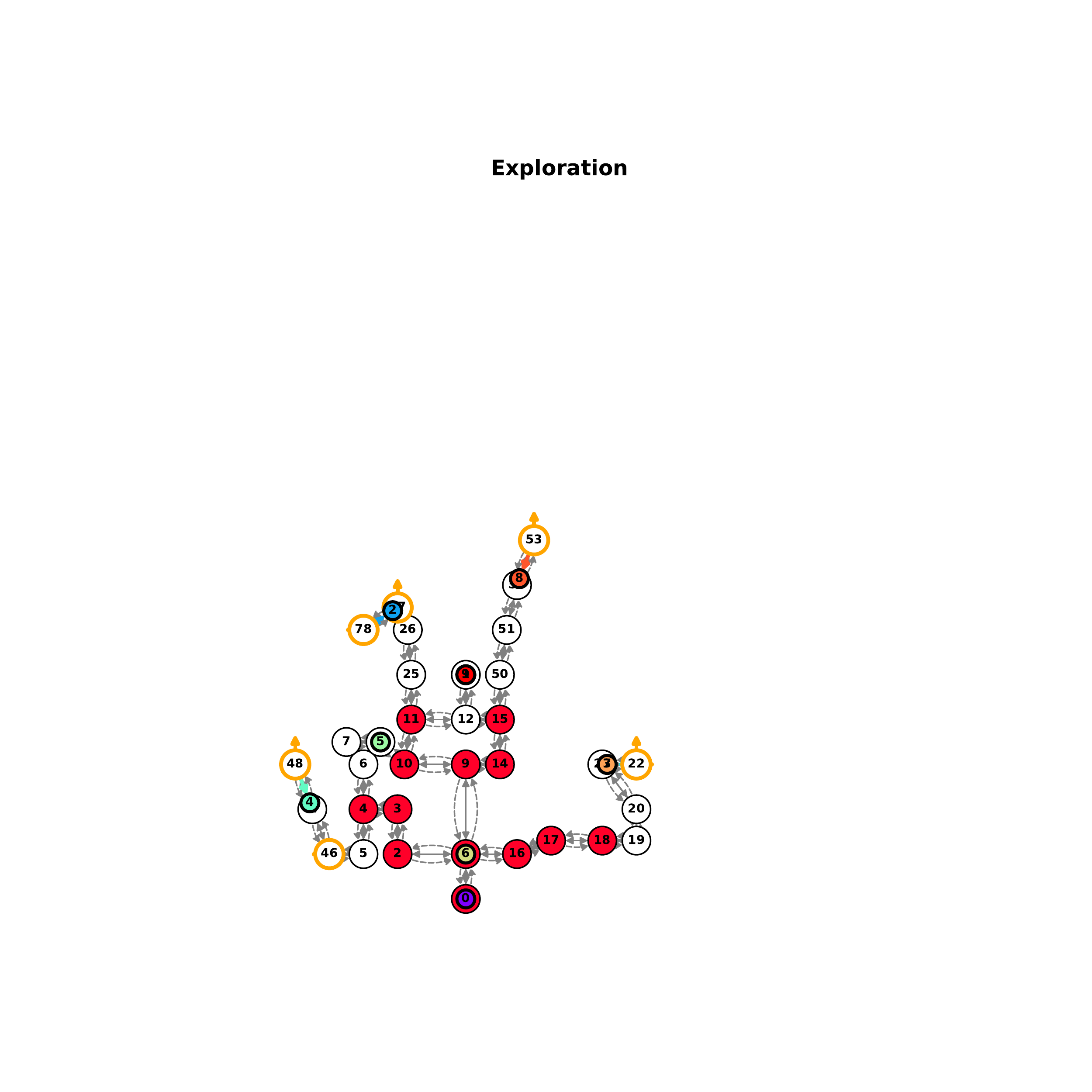}};
        \node at (0.04\textwidth, 2.1) {$t=20$};

        \node at (0.25\textwidth, 0) {\includegraphics[width=0.3\textwidth,trim={3cm 3cm 3cm 4.2cm},clip]{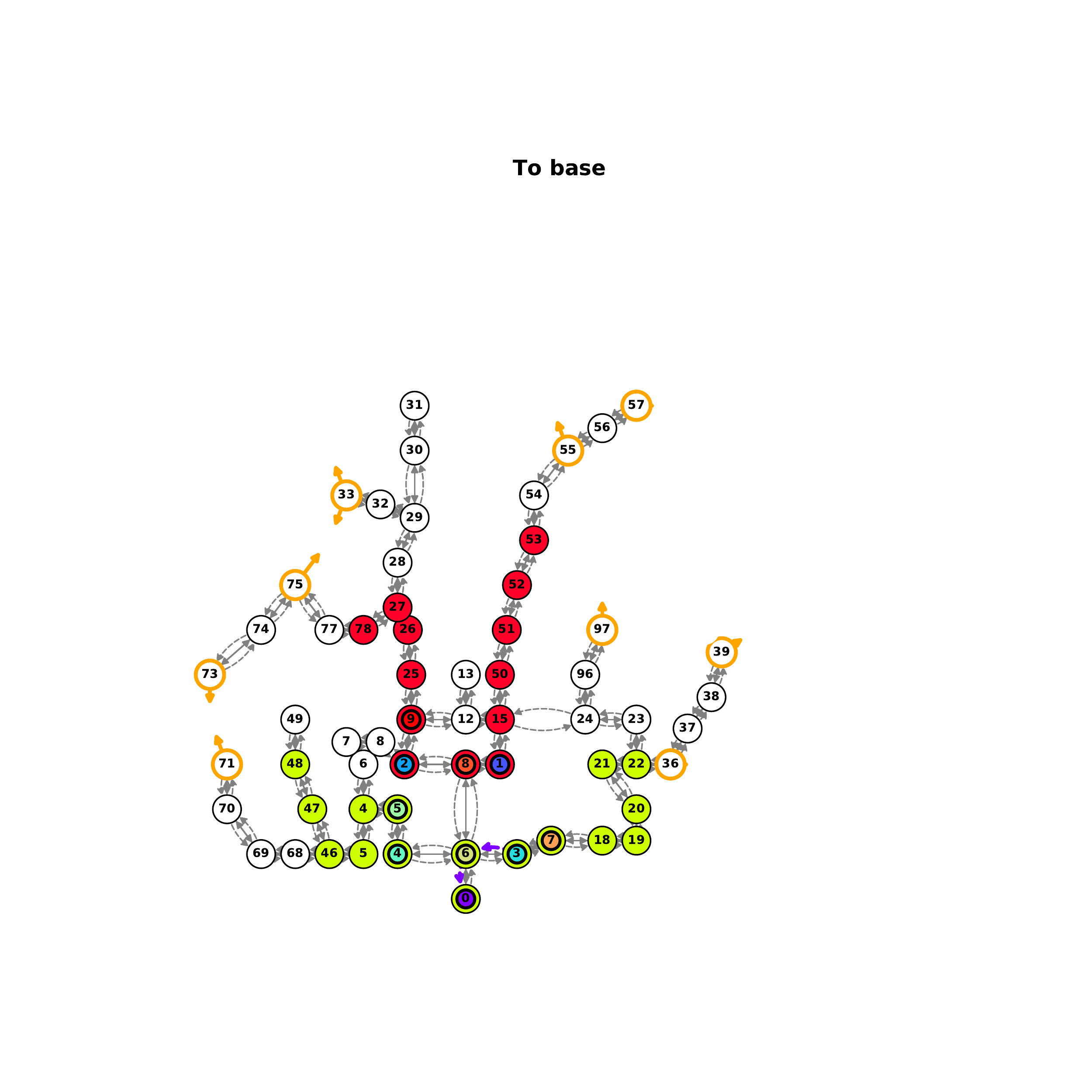}};
        \node at (0.29\textwidth, 2.1) {$t=50$};

        \node at (0.50\textwidth, 0) {\includegraphics[width=0.3\textwidth,trim={3cm 3cm 3cm 4.2cm},clip]{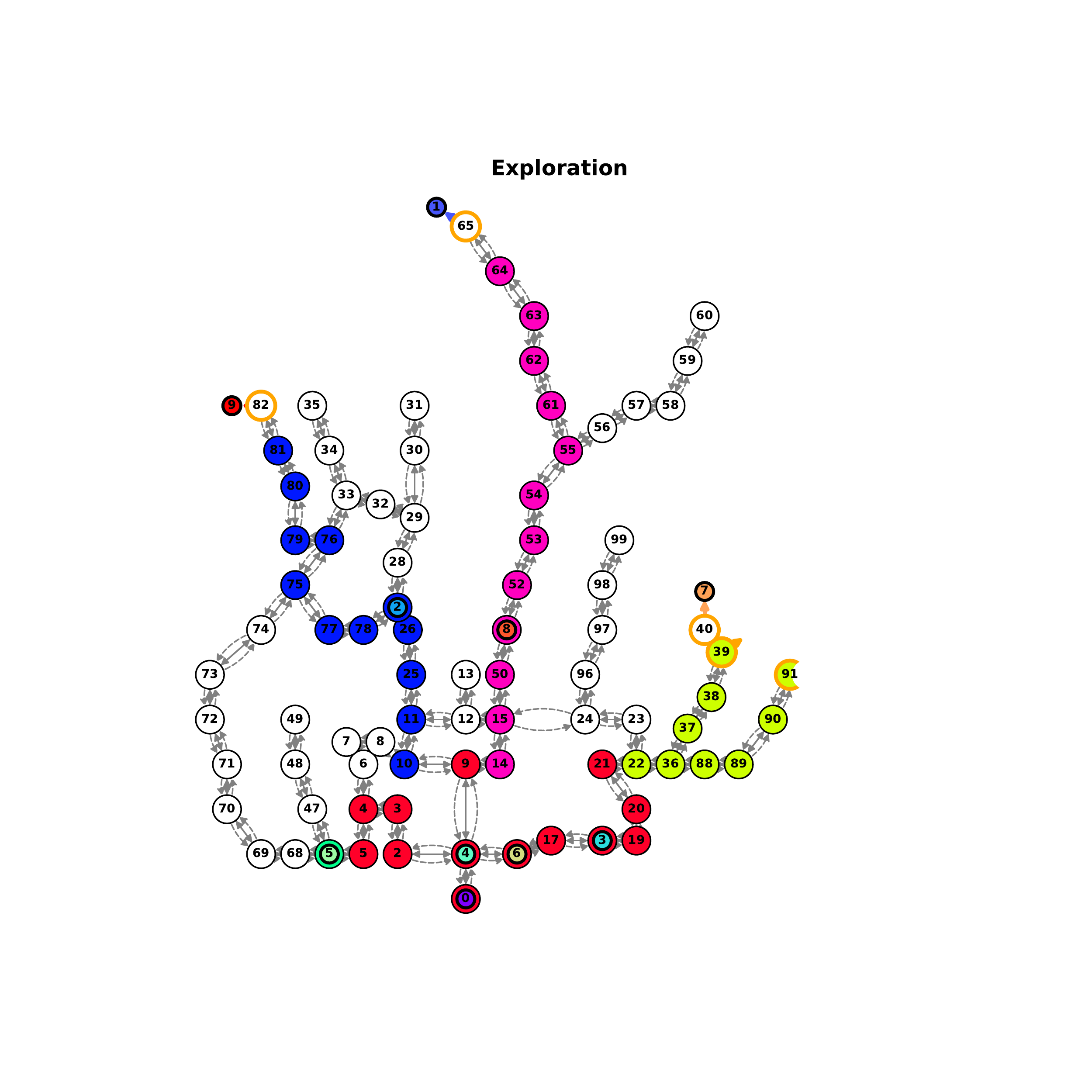}};
        \node at (0.54\textwidth, 2.1) {$t=100$};

        \node at (0.75\textwidth, 0) {\includegraphics[width=0.3\textwidth,trim={3cm 3cm 3cm 4.1cm},clip]{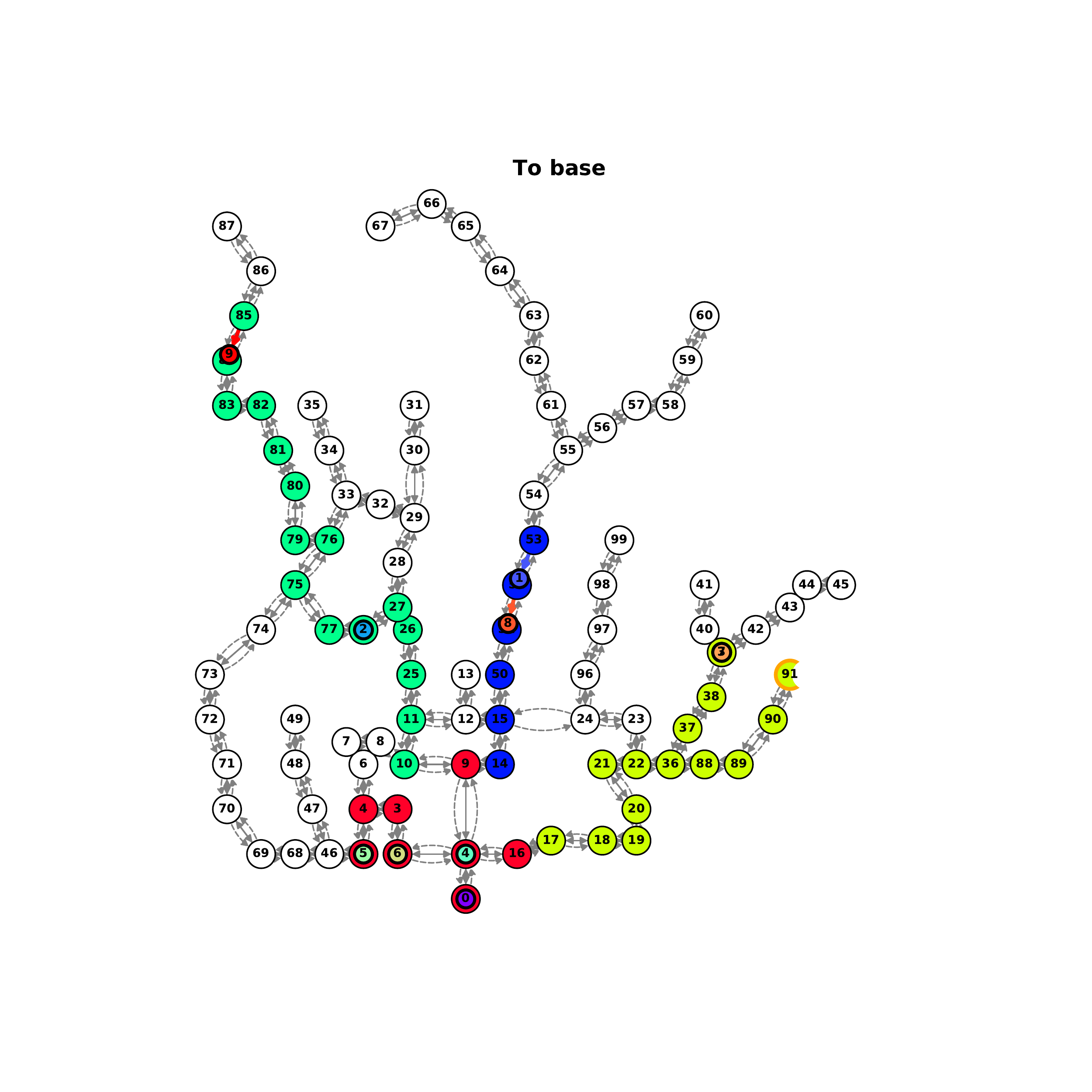}};
        \node at (0.79\textwidth, 2.1) {$t=150$};

    \end{tikzpicture}    
    \vspace{-9mm}
    \caption{Simulation of intermittent connectivity constraint in an exploration setting using 10 agents in a network with 100 states. The color of the states indicates the cluster they belong to. Agent 0 in state 0 is a static base station acting as master and is systematically updated about the progress between each exploration cycle.}
    \label{fig:huge_exploration}
    \vspace{-1mm}
\end{figure*}
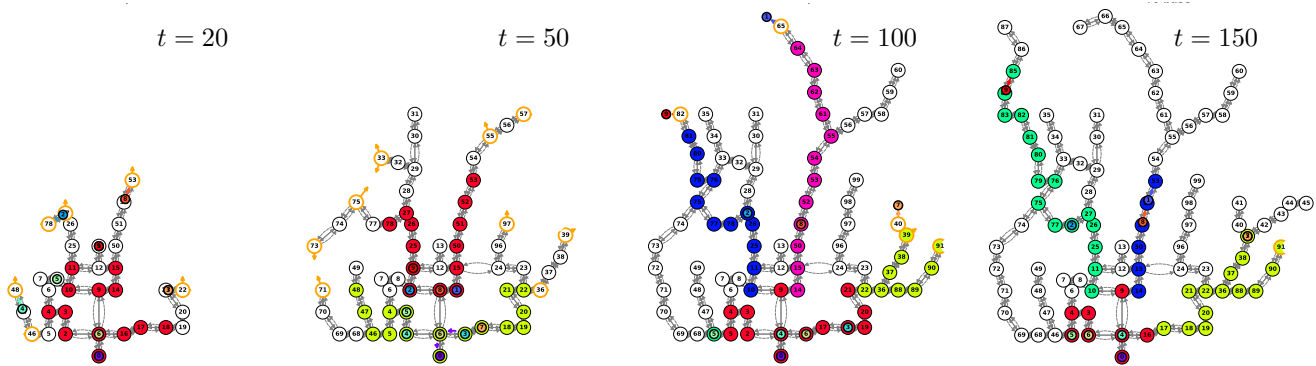

\subsection{Exploration via Clustering}

The clustering method presented in Section \ref{sec:clustering} allows us to solve much larger problems than what is possible with a single intermittent connectivity problem. In addition, further scalability improvements are possible via heuristics that disregard states that are unlikely to be used (dead states), and evacuating agents in clusters that have no child cluster or frontiers (dead clusters). As an example we consider exploration via the strategy in Fig. \ref{fig:blockdia_plan} of a network with 100 states, using 10 agents including a static base station as master, with initially just one state known. Exploration of the whole environment required computing and executing nine pre-exploration and post-exploration plans, resulting in a total of 48 solved intermittent connectivity problems. The computation time in any cluster takes no more than 1.82 seconds which is negligible in networks of this size where real-world execution of a single plan can take several minutes.

The time evolution of the exploration can be seen in Fig. \ref{fig:huge_exploration} that shows how the known set of states is expanded\footnote{Full simulation available at \url{https://youtube.com/watch?v=VQjEOFJNRjk}.}.
The color of a state indicates the cluster the state belong to. States that are not used in the optimization problem (dead states) are shown as white. Agents are well distributed in the network and simultaneously explore different parts of the network.

\section{Conclusion}
\label{sec:conclusion}

In this article we studied intermittent connectivity planning in multi-robot systems. Motivated by information constraints in applications, we proposed information-consistent intermittent connectivity to handle situations where the plan itself is part of the information that should be distributed. As a solution to this problem we presented an Integer Linear Program based on ideas from the literature on network flows, and showed that it scales better than established methods. In addition, we proposed novel ``master constraints'' that ensure that the resulting plan is information-consistent.

Motivated by utilizing intermittent connectivity for exploration, we presented an exploration strategy based on repeated solutions of information-consistent intermittent connectivity problems. To further improve scalability of this strategy we proposed a clustering method that decomposes a large intermittent connectivity problem for exploration into smaller problems that can be solved independently, and demonstrated in simulation how the resulting algorithm is capable of efficiently exploring a 100-node network using a team of 10 robots.

Current work includes real-world implementation of the exploration algorithm for use in the DARPA Subterranean Challenge by NASA Jet Propulsion Laboratory. 









\addtolength{\textheight}{-14cm}   

\bibliographystyle{IEEEtran}
\bibliography{IEEEabrv,refs}

\end{document}